\documentclass{article}

\usepackage{microtype}
\usepackage{graphicx}
\usepackage{multirow}
\usepackage{subfigure}
\usepackage{subcaption}
\usepackage{wrapfig}
\usepackage{booktabs} % for professional tables
\usepackage{amsmath}
\usepackage{amssymb}
\usepackage{mathtools}
\usepackage{amsthm}
\usepackage{colortbl}

\usepackage{hyperref}

\usepackage[accepted]{icml2026}

\usepackage[capitalize,noabbrev]{cleveref}

\theoremstyle{plain}
\newtheorem{theorem}{Theorem}[section]

\newtheorem{lemma}[theorem]{Lemma}

\theoremstyle{definition}
\newtheorem{definition}[theorem]{Definition}

\theoremstyle{remark}

\usepackage[textsize=tiny]{todonotes}

\begin{document}

\twocolumn[
  \icmltitle{Adaptive Recurrent Message Passing for Test Time Computing on Graphs}



  \icmlsetsymbol{equal}{*}

  \begin{icmlauthorlist}
    \icmlauthor{Junshu Sun}{ict,ucas}
    \icmlauthor{Wanxing Chang}{ali}
    \icmlauthor{Qingming Huang}{ucas}
    \icmlauthor{Shuhui Wang}{ict}
  \end{icmlauthorlist}

  \icmlaffiliation{ict}{State Key Laboratory of AI Safety, Institute of Computing Technology, Chinese Academy of Sciences, Beijing, China}
  \icmlaffiliation{ucas}{University of Chinese Academy of Sciences, Beijing, China}
  \icmlaffiliation{ali}{DAMO Academy, Alibaba Group, Hangzhou, China}

  \icmlcorrespondingauthor{Shuhui Wang}{wangshuhui@ict.ac.cn}

  \icmlkeywords{Machine Learning, ICML}

  \vskip 0.3in
]



\printAffiliationsAndNotice{}  % no special notice (required even if empty)

\begin{abstract}
Pre-trained foundation models have demonstrated remarkable success in many domains, enabling a unified backbone to generalize across diverse downstream tasks. However, extending this paradigm to graph learning remains challenging due to the intrinsic mismatch between graph data and fixed architectural designs. In this work, we show that this limitation can be overcome via recurrent graph models. To achieve this, we conduct a systematic theoretical analysis, rigorously deriving step dependence as a necessary and sufficient condition for an adaptively convergent recurrent process. Building on this foundation, we propose AdaR, an \textbf{Ada}ptive \textbf{R}ecurrent graph model, empowering flexible test-time computing on various downstream tasks without changing model parameters. To enable adaptive inference, AdaR explicitly encodes normalized step information and representation–target relations into the recurrent updates. To ensure convergence of the recurrent process, AdaR employs gradient-based supervision signals that guide representation updates throughout the recurrence. Empirical results demonstrate that AdaR consistently outperforms strong baselines in both inductive and transductive settings.
\end{abstract}

\section{Introduction}
% \begin{figure}[htb]
\begin{figure}[htb]
\centering
\includegraphics[width=0.95\linewidth]{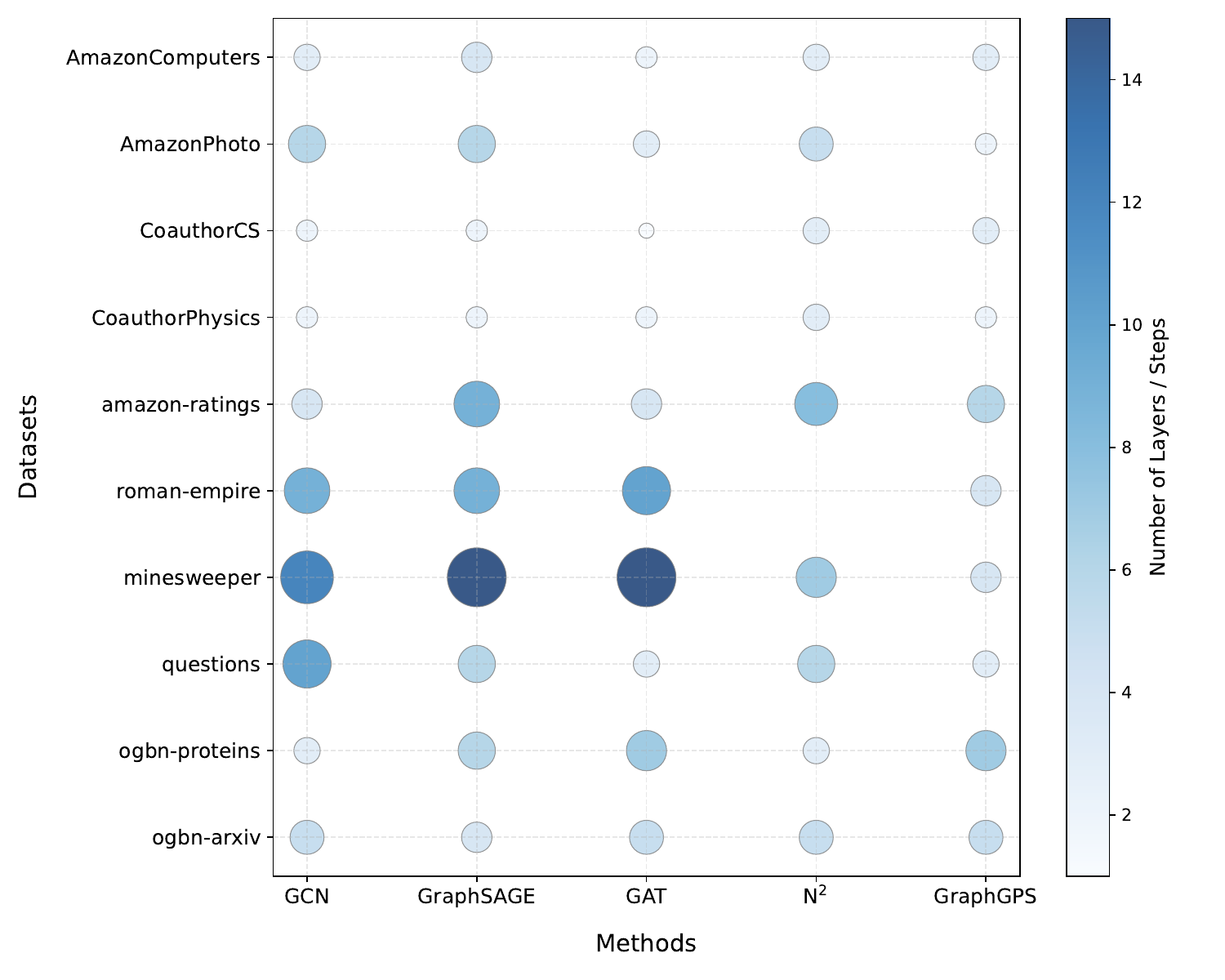}
% \vspace{-0.1in}
\caption{\textbf{Diversity of Optimal Layer Configurations Across Datasets.} Graph models exhibit varying optimal numbers of layers or steps across different datasets. This indicates that effective feature extraction requires the models to adjust their receptive field based on specific datasets. In consequence, pre-trained graph models with a fixed number of layers lack the flexibility to be adapted to diverse datasets.}\label{fig:teaser}
\vspace{-0.1in}
\end{figure}
Pre-trained foundation models have achieved remarkable success in many domains, such as natural language understanding~\cite{openai_GPT4TechnicalReport_2024} and image processing~\cite{kirillov_SegmentAnything_2023}. This success highlights the potential of a unified paradigm, where a fixed model architecture with a single set of pre-trained parameters can be used for diverse downstream tasks. 
However, this paradigm has yet to be validated in graph learning, primarily due to the mismatch between graph data characteristics and pre-defined model architectures.

Specifically, graph data exhibit diverse receptive field requirements, whereas pre-defined backbone architectures impose a fixed receptive field~\footnote{Reachable nodes via message passing on the input graph.}. From the data perspective, real-world graphs vary significantly in scale, requiring models to adapt their receptive fields to extract both local neighborhood information and global graph patterns (Fig.~\ref{fig:teaser}\footnote{Results are averaged on configurations with the same number of layers to exclude variations caused by other hyperparameters.}). From the architectural perspective, graph neural networks (GNNs) expand their receptive fields via layer stacking, where multi-hop feature extraction is tightly coupled with the number of layers~\cite{kipf_SemiSupervisedClassificationGraph_2017,sun_AllinARow_2023}. As a result, pre-trained GNNs with fixed architectures can not adjust their receptive field without altering the model parameters. This mismatch between the adaptive receptive field requirements of graph data and the rigid design of pre-defined architectures poses a fundamental obstacle to realizing one-for-all pre-training on graphs.

To address this problem, previous methods introduce additional adjustments to adapt pure graph models to downstream tasks, either at the data level through graph prompting~\cite{sun_AllOneMultiTask_2023,sun_GPPTGraphPretraining_2022,fang_UniversalPromptTuning_2023} or at the model level via fine-tuning~\cite{gui_GAdapterStructureawareParameterefficient_2024,li_AdapterGNNParameterEfficientFineTuning_2024}. Other methods~\cite{chen_LLaGALargeLanguage_2024,gao_GraphWorth$K$_2024} further incorporate external knowledge by leveraging language models (LMs) to enhance task adaptability. Nevertheless, it remains an open problem whether a pre-defined graph backbone can be directly adapted to downstream tasks without any additional architectural or data modifications.

In this paper, we answer this question affirmatively through recurrent graph models, which can adjust their receptive fields by computing with different numbers of recurrent steps during test time. We present a systematic study of recurrent graph models, establishing step dependence as a necessary and sufficient condition for adaptive inference under different recurrent step budgets. Based on this theoretical foundation, we propose an \textbf{Ada}ptive \textbf{R}ecurrent graph model, termed \textbf{AdaR}, enabling flexible test-time computing on graphs with varying receptive field requirements without altering the model architecture.

To enable an adaptive recurrent process, AdaR explicitly depends on the current \textit{step information}, the current graph node representation, and the \textit{representation-target relations}. Among these, the \textit{step dependence} ensures convergent approximation towards the targets based on our theoretical analysis. By explicitly encoding normalized step information into the recurrent updates, the model is aware of its recurrent progress and avoids out-of-distribution behavior on unseen budgets for the number of recurrent steps. The \textit{relation dependence} enables adaptive inference based on specific inputs and targets. AdaR dynamically computes relations between the current representations and their targets, adjusting the representation updates accordingly. During the training stage, AdaR employs gradient-based supervision signals that guide representation updates throughout the recurrence. Empirically, we demonstrate that AdaR can be trained with a small iteration budget while being evaluated with substantially larger budgets, consistently outperforming baselines on both inductive and transductive tasks. Our contributions can be summarized as follows
\begin{itemize}
    \item We present a theoretical analysis of recurrent graph models, deriving a general formulation and a necessary and sufficient condition for an adaptively convergent process.
    \item We design AdaR, an adaptive recurrent graph model, enabling flexible test-time computing on various downstream tasks without changing model parameters.
    \item We demonstrate the superiority of AdaR on transfer tasks, outperforming baselines in both inductive and transductive settings.
\end{itemize}

% \begin{figure}[htb]
\begin{figure*}[htb]
\centering
\includegraphics[width=\linewidth]{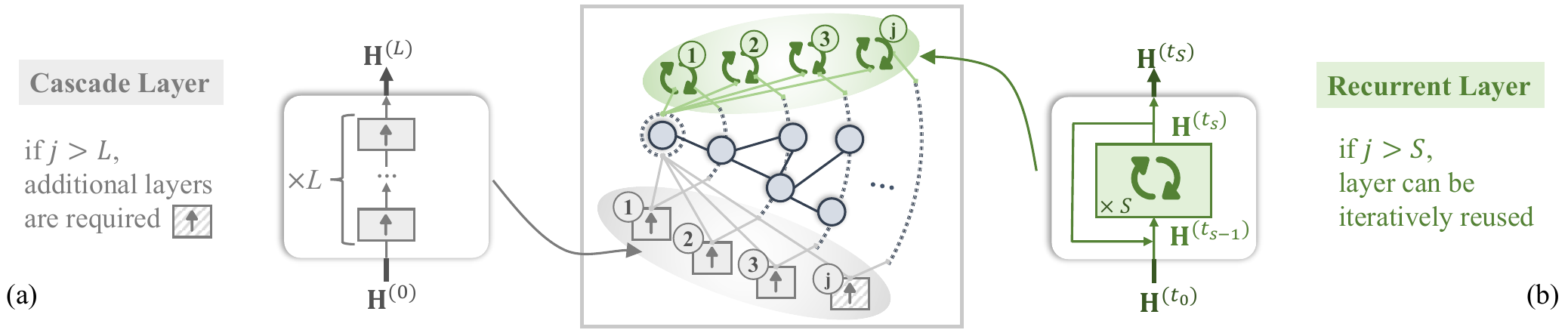}
% \vspace{-0.1in}
\caption{\textbf{Comparison Between Traditional and Recurrent Graph Model.} (a) Traditional models stack multiple cascade layers. (b) In contrast, recurrent models employ a recurrent layer, which can be iteratively reused to adapt to different receptive field requirements.}\label{fig:recurrent}
% \vspace{-0.1in}
\end{figure*}
\section{Related Work}
\subsection{Adaptation with Pure Graph Models}
\paragraph{Fine-tuning for model adjustment.}
A direct solution for adapting graph backbones to diverse downstream tasks is to fine-tune pre-trained models. Primary methods~\cite{qiu_GCCGraphContrastive_2020,hu_GPTGNNGenerativePreTraining_2020} fine-tune the entire set of model parameters. However, full fine-tuning is resource-intensive and tends to overfit when downstream tasks provide only limited labels. To overcome these limitations, parameter-efficient fine-tuning (PEFT) methods have been proposed for graph models. AdapterGNN~\yrcite{li_AdapterGNNParameterEfficientFineTuning_2024} provides a comprehensive comparison of traditional PEFT techniques in the graph learning domain and introduces dual adapter modules specifically designed for GNN architectures. G-Adapter~\yrcite{gui_GAdapterStructureawareParameterefficient_2024} further extends PEFT to graph transformers~\cite{ying_TransformersReallyPerform_2021,sun_RelievingAggregatingEffect_2025}, employing graph convolution operators to inject structural information into the tuning process.

\paragraph{Prompting for data adjustment.}
Despite their efficiency, the effectiveness of PEFT methods does not always scale favorably with increasing model size~\cite{yu_SurveyFewShotLearning_2024}. Consequently, recent studies have shifted the focus from adjusting model parameters to adapting the input data, aiming to bridge the gap between pre-training and inference. In particular, graph prompt learning proposes to adapt node features or graph structures through task-specific prompts~\cite{sun_AllOneMultiTask_2023,wang_MultiDomainGraphFoundation_2025}. For feature-level prompts, auxiliary embeddings are introduced either to unify input node features~\cite{sun_GPPTGraphPretraining_2022,sun_AllOneMultiTask_2023,fang_UniversalPromptTuning_2023} or to refine intermediate layer representations~\cite{liu_GraphPromptUnifyingPreTraining_2023,yu_MultiGPromptMultiTaskPreTraining_2024a}. For structure-level prompts, substructures such as subgraphs or pseudo nodes are learned to augment the original graph topology~\cite{sun_AllOneMultiTask_2023,tan_VirtualNodeTuning_2023,zhang_StructurePretrainingPrompt_2023,huang_PRODIGYEnablingContext_2023}. For more discussion on the pseudo nodes in adaptive graph models, please refer to Appendix~\ref{sec:app-related}.

\subsection{Adaptation with Language Models}
Language models exhibit strong generalizability across diverse downstream tasks~\cite{openai_GPT4TechnicalReport_2024}. Leveraging this capability, researchers have explored three ways to incorporate LMs in graph learning: as primary backbones, as knowledge providers for graph models, or as collaborators alongside graph models. To construct backbones with LMs, graphs are converted into sequential data through textual descriptions~\cite{li_ZeroGInvestigatingCrossdataset_2024} or structure-aware node traversals~\cite{chen_LLaGALargeLanguage_2024,gao_GraphWorth$K$_2024}. Beyond serving as backbones, LMs can also enhance graph models for downstream tasks. Specifically, some methods transfer external knowledge from LMs to graph models via strategies such as knowledge distillation~\cite{xia_OpenGraphOpenGraph_2024,pan_DistillingLargeLanguage_2024} and representation alignment~\cite{zhu_GraphCLIPEnhancingTransferability_2025}. Other methods integrate LMs and graph models, combining topological features captured by graph models with the strong generalizability of LMs~\cite{tian_GraphNeuralPrompting_2024a,zhang_GraphTranslatorAligningGraph_2024a,tang_GraphGPTGraphInstruction_2024a,huang_CanGNNBe_2024a,kong_GOFAGenerativeOneAll_2024,hu_LetsAskGNN_2024}.

\subsection{Recurrent Graph Models}
Recurrent graph models have demonstrated strong performance in various real-world applications~\cite{selsam_LearningSATSolver_2018,bresson_ResidualGatedGraph_2018}. Scarselli et al.~\yrcite{scarselli_GraphNeuralNetwork_2009} first introduced recurrent architectures to graph learning, where model parameters are constrained to satisfy contraction conditions, ensuring convergence to a fixed point via iterative updates. Following the recurrent paradigm, subsequent works relax the contraction constraints and integrate local message passing with gating mechanisms during the recurrent process~\cite{li_GatedGraphSequence_2017,ruiz_GatedGraphRecurrent_2020}. More recently, $\mathbf{N}^2$ extends recurrent models to global message passing by enabling dynamic interactions between graph nodes and learnable pseudo nodes~\cite{sun_DynamicMessagePassing_2024}. 

Despite their effectiveness, both gated methods and $\mathbf{N}^2$ assume the same fixed iteration budget for training and testing, which introduces two key limitations. First, when the optimal number of iterations is large, training incurs prohibitive memory overhead. Second, once trained, these models cannot compute with different iteration budgets, failing to satisfy the diverse receptive field requirements of downstream tasks. To empower variable iteration budgets for flexible test-time computing, we theoretically show that the model must explicitly incorporate step information, and propose AdaR to empirically validate this principle.

\section{Method}
Local message passing expands the node receptive field by limited hops per step. As a result, capturing broader graph contexts requires multiple rounds of message passing. Traditional graph models achieve this by stacking multiple cascade layers, each with its own parameters. This stacking strategy enlarges the receptive field but also increases parameter count and changes the architecture. To avoid these drawbacks, a recurrent graph model is necessary, where a single message passing layer is iteratively reused as the recurrent layer to expand the receptive field adaptively (Fig.~\ref{fig:recurrent}). In this section, we first present a general formulation of recurrent graph models, then derive the condition required for ideal recurrent behavior, and finally implement the principle in a novel recurrent model AdaR.

\subsection{Problem Setup}\label{ssec:problem-setup}
\paragraph{Notations.}
Given an input graph $\mathcal{G}=(\mathcal{V}, \mathcal{E})$. Let $\mathcal{V}=\{v_1, \dots, v_n\}$ denote the set of nodes with $|\mathcal{V}|=n$, and $\mathcal{E}=\{e_{i,j}\mid v_j \in \mathcal{N}(v_i)\}$ denote the set of edges with $|\mathcal{E}|=m$, where $\mathcal{N}(v)$ is the set of one-hop neighbors of node $v$. Each node $v\in \mathcal{V}$ is associated with a feature vector $\mathbf{x}_v \in \mathbb{R}^{d_\mathtt{in}}$, where $d_\mathtt{in}$ denotes the feature dimension. The node feature matrix is then defined as $\mathbf{X} = (\mathbf{x}_{v_1}, \dots, \mathbf{x}_{v_n})^\top \in \mathbb{R}^{n\times d_\mathtt{in}}$. Let $\mathbf{A} \in \mathbb{R}^{n\times n}$ denote the adjacency matrix of $\mathcal{G}$, where $\mathbf{A}_{i,j} = 1$ if $e_{i,j}\in \mathcal{E}$ and $0$ otherwise.

\paragraph{Graph Learning Tasks.} Let $c$ denote the number of labels or attributes in classification or regression. Each label or attribute is associated with a descriptive embedding, which is either provided as input or learned during training. Following Sun et al. \yrcite{sun_DynamicMessagePassing_2024}, given the output graph representation $\mathbf{Z}\in\mathbb{R}^{n_o\times d}$ and the target-description matrix $\mathbf{C}\in\mathbb{R}^{c\times d}$, we unify different graph learning tasks by computing the inner-product similarity $\texttt{o}(\mathbf{Z}\mathbf{C}^\top)$ as the final prediction. The function $\texttt{o}(\cdot)$ is implemented as \texttt{softmax} for classification
and the identity mapping for regression. Here, $n_o$ denotes the number of prediction targets, which equals $n$, $m+m_\texttt{neg}$, and $1$ for node-, edge-, and graph-level tasks on a single graph, respectively. $m_\texttt{neg}$ denotes the sampled negative edges which do not belong to $\mathcal{E}$.

\paragraph{Message Passing.} For both GNNs and graph transformers, the fundamental graph operators can be formulated as message passing~\cite{velickovic_MessagePassingAll_2022}. Given a central node $v_i$, the message passing result $\mathbf{h}_{v_i}^{(l)}$ at the $l$-th layer $\mathtt{f}^{(l)}$ can be formulated as
\begin{equation}\label{eq:mp}
\begin{aligned}
    \mathbf{h}_{v_i}^{(l)}&=[\mathtt{f}_{{\Theta^{(l)}}}^{(l)}(\mathbf{H}^{(l-1)},\mathbf{A})]_{v_i}\\
    &=\bigoplus_{v_j\in\mathcal{N}(v_i)}\left[\mathtt{\phi}^{(l)}(\mathbf{h}_{v_i}^{(l-1)}),\mathtt{\psi}^{(l)}(\mathbf{h}_{v_j}^{(l-1)})\right],
\end{aligned}
\end{equation}
where $\mathbf{H}^{(l)} = (\mathbf{h}_{v_1}^{(l)}, \dots, \mathbf{h}_{v_n}^{(l)})^\top \in \mathbb{R}^{n\times d}$. $\mathbf{h}_{v_i}^{(0)}=\mathbf{x}_{v_i}$. $\mathtt{\phi}^{(l)}$ and $\mathtt{\psi}^{(l)}$ denote the feature-wise transformation, such as the linear mapping. $\oplus$ can be any aggregation function permutation invariant to the ordering of nodes, such as a graph convolution operator~\cite{kipf_SemiSupervisedClassificationGraph_2017,hamilton_InductiveRepresentationLearning_2017,velickovic_GraphAttentionNetworks_2018} for message passing on the input graphs, while being a self-attention aggregator on the attention graph. $\Theta^{(l)}$ denotes the set of parameters for $\oplus$ at the $l$-th layer.

\subsection{Recurrent Graph Model} 
Let $\mathtt{F}=\mathtt{f}_\mathtt{in}\circ\mathtt{f}_\Theta\circ\cdots\circ\mathtt{f}_\Theta\circ\mathtt{f}_\mathtt{out}$ denote the whole recurrent graph model, where $\mathtt{f}_\mathtt{in}$ and $\mathtt{f}_\mathtt{out}$ are the input and output transformation, respectively. $\mathtt{f}_\Theta$ denotes the recurrent layer employed iteratively to refine the node representation. Given a graph $\mathcal{G}$, a target representation $\mathbf{C}$, and the corresponding ground-truth labels, the learning objective is to optimize the recurrent model $\mathtt{F}$ such that the inner-product relations between the output representations and the target representations are properly adjusted, thereby minimizing the prediction loss.

\paragraph{Recurrent Message Passing.}
The recurrent layer can be regarded as the discrete approximation of a continuous-time dynamical system. For time $t\in[0,T]$, let $\{\mathbf{H}^{^{*}{(t)}}\}$ denote the target trajectory of the system governed by $\frac{d}{dt}\mathbf{H}^{^{*}{(t)}}=\mathtt{v}(\mathbf{H}^{^{*}{(t)}},\mathcal{C},t)$, $\mathbf{H}^{^{*}{(0)}}=\mathbf{H}^{(0)}$, where $\mathtt{v}(\cdot)$ is a continuous mapping that assigns the current representation $\mathbf{H}^{^{*}{(t)}}$ a velocity matrix. $\mathcal{C}$ denotes the set of conditions such as the adjacent matrix $\mathbf{A}$. This process can be further discretized. Given a total recurrent step budget $S$, and step size $\Delta t=T/S$, the target discrete solution at the recurrent step $s$ is $\mathbf{H}^{^{*}{(t_s)}}$, where $t_s=s\Delta t,s=0,1,\cdots,S$. Following the same form, Eq.~\ref{eq:mp} can be reformulated at the $s$-th step as
\begin{equation}\label{eq:mp-recurrent}
\begin{aligned}
    \mathbf{h}_{v_i}^{(t_s)}
    &=[\mathtt{f}_\Theta(\mathbf{H}^{(t_{s-1})},\mathcal{C}, t_{s-1})]_{v_i}\\
    &=\bigoplus_{v_j\in\mathcal{N}(v_i)}\left[\phi(\mathbf{h}_{v_i}^{(t_{s-1})}),\psi(\mathbf{h}_{v_j}^{(t_{s-1})})\right].
\end{aligned}
\end{equation}
Different from traditional message passing, $\mathtt{f}_\Theta$ is employed recursively, with the parameter set $\Theta$ shared across the message-passing steps. A recurrent message-passing process aims to produce trajectories that stay close to the continuous reference $\{\mathbf{H}^{^{*}{(t_s)}}\}^{S}_{s=0}$ at every step, which can be formally defined as uniform convergence.
\begin{definition}[Uniform Convergence]
Recall that for time $t\in[0,T]$, a step budget $S$, and step size $\Delta t = T/S$, the target discrete solution at the recurrent step $s$ is $\mathbf{H}^{^{*}(t_s)}$ with $t_s = s\Delta t$. Let $\{\mathbf{H}^{(t_s)}\}^{S}_{s=0}$ be the discrete trajectory generated by a recurrent layer. We say that ${\mathbf{H}^{(t_s)}}$ converges uniformly to ${\mathbf{H}^{^{*}(t_s)}}$ if
\begin{equation}\label{eq:uni-conv}
    \sup_{0\le s\le S}\|\mathbf{H}^{(t_s)}-\mathbf{H}^{^{*}{(t_s)}}\|_\mathtt{F}\le C\Delta t,
\end{equation}
for some constant $C$ independent of $S$. $\|\cdot\|_\mathtt{F}$ denotes the Frobenius norm. The approximation error in uniform convergence vanishes uniformly over the entire time range. As $S\to\infty$, $C\Delta t\to 0$.
\end{definition}

% \begin{figure}[htb]
\begin{figure*}[htb]
\centering
\includegraphics[width=0.98\linewidth]{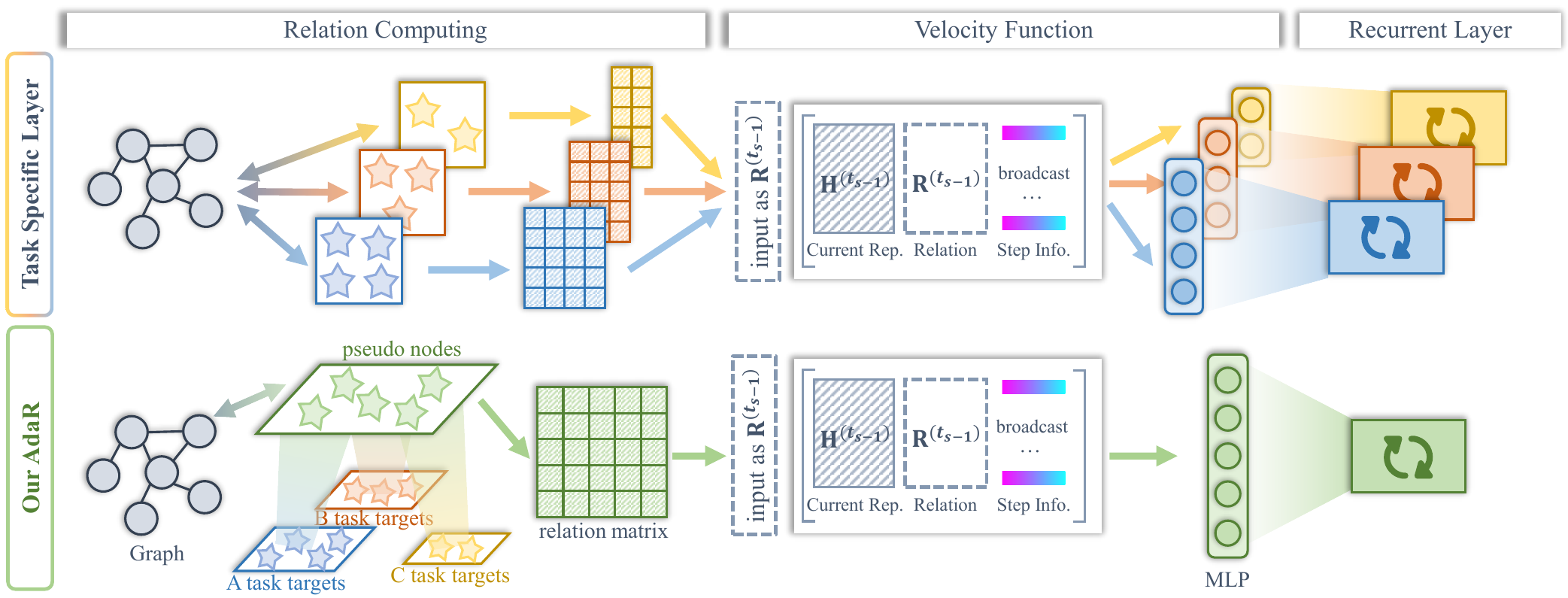}
% \vspace{-0.1in}
\caption{\textbf{Adaptability of AdaR.} To enable a target-aware recurrent process, AdaR introduces relations with targets to compute representation updates. However, the number of targets varies across tasks. To tackle this problem, a proxy-based relation computing strategy is designed, where a fixed number of pseudo nodes aggregate information from different numbers of targets, serving as target proxies for relation computing.}\label{fig:main}
% \vspace{-0.1in}
\end{figure*}
\paragraph{Step Information.}
The mapping in the target continuous system assigns a different velocity matrix at each step. This indicates that the recurrent process in Eq.~\ref{eq:mp-recurrent} should depend on the step information $t_s=s\Delta t$ to achieve uniform convergence to specific target systems. Formally, we show that a recurrent layer without step information cannot converge to the target system uniformly, while a layer with step information can. For the target trajectory to be meaningful in the sense of convergence, we assume it is generated by a continuously differentiable Lipschitz function. This assumption is mild and non-restrictive, as it is a standard condition ensuring the well-posedness of trajectory dynamics.
\begin{lemma}[Non-uniform Convergence without Step Information]\label{thrm:non-uni-conv}
Let $\mathbf{f}_\Theta(\mathbf{H}^{(t_{s-1})},\mathcal{C})$ be a recurrent layer that generates a discrete trajectory $\{{\mathbf{H}^{(t_s)}}\}^{S}_{s=0}$ without the step information $t_s=s\Delta t$. There exist a target trajectory $\{\mathbf{H}^{^{*}(t_s)}\}^{S}_{s=0}$ generated by a continuously differentiable Lipschitz function and a constant $\varepsilon>0$ such that, for sufficiently large $S$, $\sup_{0\le s\le S}\|\mathbf{H}^{(t_s)}-\mathbf{H}^{^{*}{(t_s)}}\|_\mathtt{F}\ge\varepsilon$.
\end{lemma}
\begin{theorem}[Uniform Convergence with Step Information]\label{thrm:uni-conv}
For any target trajectory $\{\mathbf{H}^{^{*}(t_s)}\}^{S}_{s=0}$ generated by a continuously differentiable Lipschitz function, there exists a recurrent layer $\mathbf{H}^{(t_{s})}=\mathbf{f}_\Theta(\mathbf{H}^{(t_{s-1})},\mathcal{C},t_{s-1})$ such that ${\mathbf{H}^{(t_s)}}$ converges uniformly to ${\mathbf{H}^{^{*}(t_s)}}$.
\end{theorem}
Please refer to Appendix~\ref{sec:app-proof} for the proof. Based on Lemma~\ref{thrm:non-uni-conv} and Theorem~\ref{thrm:uni-conv}, we can see that previous recurrent graph models fail to converge with different iteration budgets due to their miscapturing of the step information. This observation motivates the design of an adaptive recurrent graph model.

\subsection{AdaR for Flexible Test-time Computing}
In this section, we present our adaptive recurrent graph model, termed \textbf{AdaR}, following the recurrent message-passing paradigm. By explicitly encoding step information, AdaR can operate under varying iteration budgets without retraining or architectural modification, thereby enabling flexible test-time computing for different downstream tasks.

\subsubsection{Recurrent Update with the Velocity Function}
Without loss of generality, a recurrent layer accessible to step information can be written in the residual form
\begin{equation}
\begin{aligned}
\mathbf{H}^{(t_{s})}
&=\mathbf{f}_\Theta(\mathbf{H}^{(t_{s-1})},\mathcal{C},t_{s-1})\\
&=\mathbf{H}^{(t_{s-1})}+\mathtt{v}_\Theta(\mathbf{H}^{(t_{s-1})},\mathcal{C},t_{s-1})\Delta t.
\end{aligned}
\end{equation}
Here, $\mathtt{v}_\Theta(\cdot)$ computes the velocity for nodes and updates $\mathbf{H}^{(t_{s-1})}$ to adjust their relation to the targets, ensuring the optimal output-target relation at the final step. To achieve this, $\mathtt{v}_\Theta(\cdot)$ should explicitly depend on
\begin{itemize}
    \item the current representation $\mathbf{H}^{(t_{s-1})}$,
    \item the representation-target relations $\mathbf{R}^{(t_{s-1})}$, ensuring adaptive updates based on specific inputs and targets,
    \item the step information $\mathbf{S}_{s-1,\cdot}$, ensuring convergent approximation towards the targets,
\end{itemize}
where $\mathbf{S}\in\mathbb{R}^{s\times d}$ denotes the step matrix. The step velocity function $\texttt{v}_\Theta(\cdot)$ can thus be formulated in a simple form
\begin{equation}\label{eq:velocity}
\begin{aligned}
&\texttt{v}_\Theta(\mathbf{H}^{(t_{s-1})},\mathcal{C},t_{s-1})\\
&=\texttt{NL}\Big(
\underset{\texttt{current rep.}}{\underbrace{\texttt{MP}(\mathbf{H}_{\ }^{(t_{s-1})})}}\big\|
\underset{\texttt{relation}}{\underbrace{\mathbf{R}_{\ }^{(t_{s-1})}}}\big\|
\underset{\texttt{step info.}}{\underbrace{\mathbf{1}_n\mathbf{S}_{s-1,\cdot}^\top}}\Big),
\end{aligned}
\end{equation}
where $\mathcal{C}=\{\mathbf{A},\mathbf{R}^{(t_{s})},\mathbf{S}_{s,\cdot}\}$, $\texttt{NL}$ can be any non-linear transformation, such as a linear layer combined with $\texttt{ReLU}$. $\texttt{MP}$ denotes message passing function, which is implemented following $\mathbf{N}^2$~\cite{sun_DynamicMessagePassing_2024}. Additional implementation details are provided in Appendix~\ref{sec:app-mp}. $\|$ denotes the concatenation operation. Here, $\mathbf{1}_n\in\mathbb{R}^n$ is an all-one vector, such that $\mathbf{1}_n \mathbf{S}_{s-1,\cdot}^{\top}$ broadcasts the step embedding to all nodes. In the following context, we show how to prepare the last two features for computing the velocity.

\subsubsection{Encoding Step Information}
Given the step information $t_s=s\Delta t=sT/S$, two common design choices include fixing the unit time step $\Delta t$ or fixing the total time $T$. Specifically, by fixing the total time ({\it e.g.}, $T=1$), different choices of the iteration budget $S$ correspond to different discretizations of the same target trajectory, ensuring that the model is always trained to produce $\mathbf{H}^{(1)}$ as the final representation. In contrast, by fixing the unit time ({\it i.e.}, $\Delta t=1$), different iteration budgets $S$ correspond to the same discretization scheme of different trajectories. Although this formulation simplifies the learning of $\texttt{v}_\Theta$ to produce velocity within a fixed unit time, it may lead to out-of-distribution behavior when generalizing to unseen time and budgets. Based on the stability under different iteration budgets, AdaR fixes the total time $T=1$. In Appendix~\ref{sec:app-relative}, we further provide theoretical and empirical analysis to validate the usage of a fixed total time rather than a fixed unit time step.

To map step scalers into the $d$-dimensional space, we propose to adopt the sinusoidal position encoding function~\cite{vaswani_AttentionAllYou_2017}, yielding the step matrix $\mathbf{S}$ as
\begin{equation}\label{eq:sine}
\mathbf{S}_{s,j}=
\begin{cases}
    \sin{(\omega_k\cdot t_s)}, \texttt{if}\ j=2k\\
    \cos{(\omega_k\cdot t_s)}, \texttt{if}\ j=2k+1,\\
\end{cases}
\end{equation}
where $\omega_k=\frac{1}{10000^{2k/d}}$. The sinusoidal encoding provides a continuous representation that is smoothly conditioned on the step information $t_s$. Compared to learnable step embeddings, it avoids overfitting to a fixed number of iterations and naturally supports arbitrary step budgets without introducing additional parameters.

\subsubsection{Proxy-based Relation Computing}
Following the task formulation in Section~\ref{ssec:problem-setup}, the representation-target relation can be formulated as $\mathbf{H}^{(t_{s-1})}\mathbf{C}^\top$. However, the number of targets varies across different graph learning tasks, making it infeasible for $\texttt{v}_\Theta$ to model node-wise velocities using a fixed set of parameters.

To address this issue, we introduce a set of pseudo nodes as proxies for targets (Fig.~\ref{fig:main}). These pseudo nodes aggregate information from the target representation via message passing and update their own representations accordingly. As a result, the pseudo nodes can be regarded as abstractions of targets, serving as fixed-size proxies for relation computation. Let $\mathbf{P}^{(t_{s})}\in\mathbb{R}^{n_{\texttt{c}}\times d}$ denote the pseudo-node representations, and $\mathbf{P}^{(t_0)}$ be a learnable initialization. The computation of the pseudo-node representations can be formulated as
\begin{equation}\label{eq:proxy-abstract}
\begin{aligned}
&\mathbf{P}^{(t_{s})}
=\mathbf{P}^{(t_{s-1})}+\texttt{v}_\mathbf{\beta}(\mathbf{P}^{(t_{s-1})},\mathbf{S}_{s-1,\cdot},t_{s-1})\Delta t,\\
&\texttt{v}_\mathbf{\beta}(\cdot)
\!=\!\texttt{NL}\Big(\!\!\!
\underset{\texttt{current rep.}}
{\underbrace{\texttt{AGG}(\mathbf{C}_{\ })}}\!\!\big\|
\underset{\texttt{relation}}
{\underbrace{\mathbf{P}_{\ }^{(t_{s-1})}\mathbf{P}_{\ }^{(t_{s-1})\top}}}\big\|\!\!
\underset{\texttt{step info.}}
{\underbrace{\mathbf{1}_n\mathbf{S}_{s-1,\cdot}^\top}}\!\!\Big).
\end{aligned}
\end{equation}
The update process follows the velocity function in Eq.~\ref{eq:velocity}, with the aggregation function $\texttt{AGG}$ implemented following $\mathbf{N}^2$~\cite{sun_DynamicMessagePassing_2024}. Further details are presented in Appendix~\ref {ssec:app-mp-g}. In consequence, the relation between current graph node representations and targets can be formulated as
\begin{equation}
    \mathbf{R}^{(t_{s-1})}=\mathbf{H}^{(t_{s-1})}\mathbf{P}^{(t_{s})\top}.
\end{equation}
In practice, to ensure the adaptability to different targets, we also recurrently update the target representation as $\mathbf{C}^{(t_s)}$. Similar to Eq.~\ref{eq:velocity}, the update of targets also depends on their current representation, step information, and their relation with graph nodes. For brevity and to avoid redundancy, we provide the details in Appendix~\ref {ssec:app-mp-v}.

\begin{table*}
\caption{\textbf{Evaluation Results in the Zero-shot Inductive Setting (measured by accuracy: \%).} $\dagger$ denotes the average number of parameters among the best model configurations on each dataset.}
\label{tab:main-zero}
\centering
\begin{sc}
\begin{small}
\resizebox{0.9\linewidth}{!}{
\begin{tabular}{llrrcccccc} 
\hline
       &    & \begin{tabular}[r]{@{}r@{}}Trainable\\Params\end{tabular}  & \begin{tabular}[r]{@{}r@{}}Total\\Params\end{tabular} & arxiv-year       & WikiCS     & SportsFit  & Cora       & CiteSeer         & DBLP              \\ 
\hline
\multirow{4}{*}{\begin{tabular}[c]{@{}l@{}}GNN\\Backbones\end{tabular}} & GCN  &   100\% & 0.4M     & 21.79$_{\color{gray}\pm2.66}$       & 27.21$_{\color{gray}\pm0.23}$ & 11.48$_{\color{gray}\pm0.16}$ & 37.91$_{\color{gray}\pm0.74}$ & 35.74$_{\color{gray}\pm2.42}$       & 51.69$_{\color{gray}\pm0.36}$        \\
& GAT     &    100\% & 0.4M      & 18.96$_{\color{gray}\pm2.58}$       & 31.33$_{\color{gray}\pm0.19}$ & 11.95$_{\color{gray}\pm0.52}$ & 39.65$_{\color{gray}\pm0.71}$ & 36.52$_{\color{gray}\pm0.15}$       & 55.93$_{\color{gray}\pm1.59}$        \\
& $\mathbf{N^2}$    &   100\% & 0.4M        & 13.80$_{\color{gray}\pm0.49}$        & 22.39$_{\color{gray}\pm0.93}$ & 7.87$_{\color{gray}\pm0.97}$  & 29.30$_{\color{gray}\pm0.27}$  & 22.26$_{\color{gray}\pm1.05}$       & 36.19$_{\color{gray}\pm1.81}$        \\
& GraphSAGE  &   100\% & 0.6M   & 18.41$_{\color{gray}\pm0.86}$       & 30.44$_{\color{gray}\pm0.97}$ & 12.20$_{\color{gray}\pm0.59}$  & 36.46$_{\color{gray}\pm0.14}$ & 33.23$_{\color{gray}\pm1.09}$       & 54.15$_{\color{gray}\pm1.72}$        \\
& GCNII  &   100\% & 0.5M   & 25.59$_{\color{gray}\pm0.52}$       & 11.57$_{\color{gray}\pm1.46}$ & 13.33$_{\color{gray}\pm0.75}$  & 37.04$_{\color{gray}\pm1.48}$ & 23.82$_{\color{gray}\pm0.42}$       & 37.37$_{\color{gray}\pm0.79}$        \\
\hline
\multirow{2}{*}{\begin{tabular}[c]{@{}l@{}}Self-supervised\\Methods\end{tabular}} & DGI     &   100\% & 0.4M      & OOM              & $\ \ $8.26$_{\color{gray}\pm1.42}$  & OOM   & 15.47$_{\color{gray}\pm1.07}$ & 26.81$_{\color{gray}\pm0.71}$       & 36.78$_{\color{gray}\pm1.99}$        \\
& GRACE   &   100\% & 0.4M      & OOM              & $\ \ $7.26$_{\color{gray}\pm0.61}$  & OOM   & 19.68$_{\color{gray}\pm2.76}$ & 25.39$_{\color{gray}\pm0.70 }$       & 26.46$_{\color{gray}\pm1.58}$        \\
\hline
\multirow{8}{*}{\begin{tabular}[c]{@{}l@{}}Foundation\\Models\end{tabular}} & Vicuna-7B  &   $\ - \ $ & 6.7B    &  $-$ & 29.31$_{\color{gray}\pm0.49}$ & 30.15$_{\color{gray}\pm0.16}$ & 14.39$_{\color{gray}\pm0.74}$ &  $-$ &  $-$  \\
& LLaGA   &   0.02\% & 6.7B      & 21.20$_{\color{gray}\pm0.91}$        & $\ \ $2.65$_{\color{gray}\pm0.13}$  & $\ \ $5.45$_{\color{gray}\pm0.67}$  & 14.05$_{\color{gray}\pm0.57}$ & 15.62$_{\color{gray}\pm0.15}$       & 11.55$_{\color{gray}\pm0.89}$        \\
& OFA    &   100\% & 5.1M       & $\ \ $4.67$_{\color{gray}\pm0.14}$        & 36.89$_{\color{gray}\pm1.37}$ & 12.65$_{\color{gray}\pm1.41}$ & 39.45$_{\color{gray}\pm1.14}$ & 55.03$_{\color{gray}\pm2.09}$       & 51.01$_{\color{gray}\pm1.82}$        \\
& GFT    &   100\% & 0.4M$\dagger$       & 10.59$_{\color{gray}\pm1.61}$        & $\ \ $6.77$_{\color{gray}\pm1.09}$ & $\ \ $8.71$_{\color{gray}\pm1.49}$ & 25.15$_{\color{gray}\pm0.24}$ & 18.62$_{\color{gray}\pm0.45}$       & 31.75$_{\color{gray}\pm0.40}$        \\
& MDGFM   &   100\% & 0.4M$\dagger$        & 23.73$_{\color{gray}\pm1.15}$        & 22.87$_{\color{gray}\pm0.52}$ & 10.37$_{\color{gray}\pm0.72}$ & 10.66$_{\color{gray}\pm1.39}$ & 17.46$_{\color{gray}\pm0.87}$       & 20.84$_{\color{gray}\pm2.75}$        \\
& RiemannGFM &   100\% & 0.4M   & $\ \ $1.01$_{\color{gray}\pm0.01}$        & $\ \ $4.26$_{\color{gray}\pm0.49}$  & $\ \ $3.83$_{\color{gray}\pm0.02}$  & $\ \ $1.19$_{\color{gray}\pm0.05}$  & 11.76$_{\color{gray}\pm0.34}$       & 38.68$_{\color{gray}\pm0.37}$        \\
& UniGTE   &   0.54\% & 13.7B     &  $-$ & 34.91$_{\color{gray}\pm0.42}$ & 29.89$_{\color{gray}\pm0.15}$ & 21.29$_{\color{gray}\pm0.69}$ &  $-$ &  $-$  \\ 
% \hline
& \textbf{AdaR (Ours)} &   100\% & 0.4M  & \textbf{36.15$_{\color{gray}\pm0.52}$}       & \textbf{37.81$_{\color{gray}\pm0.94}$} & \textbf{31.56$_{\color{gray}\pm0.61}$} & \textbf{41.28$_{\color{gray}\pm0.82}$} & \textbf{59.39$_{\color{gray}\pm1.08}$}       & \textbf{59.83$_{\color{gray}\pm1.02}$}        \\
\hline
\end{tabular}
}
\end{small}
\end{sc}
\end{table*}
\subsection{Training Strategy}
Training a recurrent model requires supervision at every recurrent step. To achieve this, we propose to use the gradient of task loss with respect to the intermediate representations $\mathbf{H}^{(t_s)}$ as supervision signals, and align it with the predicted step displacement $-\texttt{v}_\Theta(\mathbf{H}^{(t_{s-1})},\mathcal{C},t_{s-1})\Delta t$. Let $\mathcal{L}_{\texttt{task}}(\mathbf{Z})$ denote a task-dependent loss computed on the representation $\mathbf{Z}$. The overall optimization objective is defined as
\begin{equation}\label{eq:loss}
\begin{aligned}
&\mathcal{L}
=\mathcal{L}_\texttt{step}+\mathcal{L}_\texttt{full}+\mathcal{L}_\texttt{task}(\mathbf{H}^{(t_{S})}),\\
&\mathcal{L}_\texttt{step}
=\frac{1}{S}\sum_{s=1}^{S}
\left\|
\underset{\texttt{target grad.}}
{\underbrace{\nabla\mathbf{H}_{\ _{\ _{\ _{\ _{\ }}}}}^{(t_{s-1})}}}
\!-\!
\underset{\texttt{pred. grad.}}
{\underbrace{\!\left(\!-\Delta\mathbf{H}^{(t_{s})}\right)}}\right\|_1,\\
&\mathcal{L}_\texttt{full}
=\left\|
\underset{\texttt{target grad.}}
{\underbrace{\nabla\mathbf{H}_{\ _{\ _{\ _{\ _{\ }}}}}^{(t_0)}}}
-
\underset{\texttt{pred. grad.}}
{\underbrace{\left(\mathbf{H}^{(t_{0})}-\mathbf{H}^{(t_{S})}\right)}}\right\|_1,\\
&\nabla\mathbf{H}^{(t_s)}
=\frac{\partial\mathcal{L}_\texttt{task}(\mathbf{H}^{(t_{s})})}{\partial\mathbf{H}^{(t_{s})}},\\
&\Delta\mathbf{H}^{(t_{s})}
\!=\!\texttt{v}_\Theta(\mathbf{H}^{(t_{s-1})},\mathcal{C},t_{s-1})\Delta t=\mathbf{H}^{(t_{s})}-\mathbf{H}^{(t_{s-1})}.\\
\end{aligned}
\end{equation}
Here, $\mathcal{L}_\texttt{step}$ aligns the negative predicted step displacement $-\Delta \mathbf{H}^{(t_s)}$ with the task gradient $\nabla\mathbf{H}^{(t_{s-1})}$ at each recurrent step, encouraging every recurrent update to move the representation in a descent direction of $\mathcal{L}_\texttt{task}$. $\mathcal{L}_\texttt{full}$ enforces global consistency by matching the accumulated displacement of the entire recurrent process with the task gradient at the initial representation. $\mathcal{L}_{\texttt{task}}$ is directly optimized to ensure task performance and guide the learning process.

\section{Experiment}
In this section, we evaluate AdaR on transfer tasks under both inductive and transductive settings\footnote{\url{https://github.com/sunjss/AdaR}}, verifying its effectiveness in handling the varying receptive-field requirements at the graph level and the node level, respectively. Graph-level diversity stems from the varying optimal model depths across graphs (Fig.~\ref{fig:teaser}). For the node-level diversity, different nodes from the same graph may require different receptive fields. This is especially common in heterophilic graphs, where informative nodes can be located several hops away from the target node, while in homophilic graphs, useful nodes are often local.

\subsection{Zero-shot Inductive Transfer}
\paragraph{Experimental Setup.}
Both AdaR and baseline methods are jointly trained on four datasets (arXiv, bookhistory, amazon-ratings, and PubMed)~\cite{chen_TextspaceGraphFoundation_2024}. After training, the models are evaluated in the zero-shot setting on downstream datasets (arxiv-year, WikiCS, SportsFit, Cora, CiteSeer, and DBLP)~\cite{chen_TextspaceGraphFoundation_2024}. Statistics of the datasets are summarized in Tab.~\ref{tab:statistics}.

\begin{table*}
    \caption{\textbf{Evaluation Results in the Supervised Transductive Setting (measured by ROC-AUC except accuracy for amz-ratings: \%).}}
    \label{tab:main-trans}
    % \vskip -0.1in
    \subtable[Heteropilic Graphs\label{tab:het-node}]{
        \centering
        \resizebox{0.49\textwidth}{!}{
        % \begin{table*}
% \caption{\textbf{Evaluation Results with Heterophlic Datasets under Supervised Inductive Setting (measured by ROC-AUC except accuracy for amazon-ratings: \%).}}
% \label{tab:main-het}
\centering
\begin{sc}
\begin{small}
\begin{tabular}{lcccc} 
\hline
           & questions  & \begin{tabular}[c]{@{}c@{}}amz \\-raings\end{tabular} & tolokers   & \begin{tabular}[c]{@{}c@{}}mine\\sweeper\end{tabular}  \\ 
\hline
GCN        & 76.09$_{\color{gray}\pm1.27}$ & 48.70$_{\color{gray}\pm0.63}$     & 83.64$_{\color{gray}\pm0.67}$ & 89.75$_{\color{gray}\pm0.52}$   \\
GAT        & 77.43$_{\color{gray}\pm1.20}$ & 49.09$_{\color{gray}\pm0.63}$     & 83.70$_{\color{gray}\pm0.47}$ & 92.01$_{\color{gray}\pm0.68}$   \\
GCNII        & 78.62$_{\color{gray}\pm0.21}$ & 48.44$_{\color{gray}\pm0.15}$     & 82.68$_{\color{gray}\pm1.24}$ & 92.94$_{\color{gray}\pm1.12}$   \\
GPRGNN     & 55.48$_{\color{gray}\pm0.91}$ & 44.88$_{\color{gray}\pm0.34}$     & 72.94$_{\color{gray}\pm0.97}$ & 86.24$_{\color{gray}\pm0.61}$   \\
H$_2$GCN      & 63.59$_{\color{gray}\pm1.46}$ & 36.47$_{\color{gray}\pm0.23}$     & 73.35$_{\color{gray}\pm1.01}$ & 89.71$_{\color{gray}\pm0.31}$   \\
FAGCN      & 77.24$_{\color{gray}\pm1.26}$ & 44.12$_{\color{gray}\pm0.30}$     & 77.75$_{\color{gray}\pm1.05}$ & 88.17$_{\color{gray}\pm0.73}$   \\
GloGNN     & 65.74$_{\color{gray}\pm1.19}$ & 36.89$_{\color{gray}\pm0.14}$     & 73.39$_{\color{gray}\pm1.17}$ & 51.08$_{\color{gray}\pm1.23}$   \\
Graphormer & OOM        & OOM            & OOM        & OOM          \\
GraphGPS   & OOM        & OOM            & 84.70$_{\color{gray}\pm0.56}$ & 92.29$_{\color{gray}\pm0.61}$   \\
Exphormer  & 73.04$_{\color{gray}\pm0.58}$ & 49.37$_{\color{gray}\pm0.36}$     & 84.20$_{\color{gray}\pm0.22}$ & 90.42$_{\color{gray}\pm0.10}$   \\
$\mathbf{N}^2$         & 77.86$_{\color{gray}\pm0.56}$ & 49.70$_{\color{gray}\pm0.49}$     & 84.85$_{\color{gray}\pm0.53}$ & 93.97$_{\color{gray}\pm0.27}$   \\ 
\hline
\textbf{AdaR (Ours)}       & \textbf{79.27$_{\color{gray}\pm0.51}$} & \textbf{49.98$_{\color{gray}\pm0.12}$}     & \textbf{84.92$_{\color{gray}\pm0.65}$} & \textbf{94.08$_{\color{gray}\pm0.18}$}   \\
\hline
\end{tabular}
\end{small}
\end{sc}
% \end{table*}
}
    }%
    \subtable[Homophilic Graphs\label{tab:hom-node}]{
        \centering
        \resizebox{0.5\textwidth}{!}{
        % \begin{table*}
% \caption{\textbf{Evaluation Results with Homophilic Datasets under Supervised Inductive Setting (measured by accuracy: \%).}}
% \label{tab:main-hom}
\centering
\begin{sc}
\begin{small}
\begin{tabular}{lcccc} 
\hline
           & \begin{tabular}[c]{@{}c@{}}Coauthor\\CS\end{tabular} & \begin{tabular}[c]{@{}c@{}}Coauthor\\Physics\end{tabular} & \begin{tabular}[c]{@{}c@{}}Amz\\Photo\end{tabular} & \begin{tabular}[c]{@{}c@{}}Amz\\Computers\end{tabular}  \\ 
\hline
GCN        & 92.92$_{\color{gray}\pm0.12}$ & 96.18$_{\color{gray}\pm0.07}$      & 92.70$_{\color{gray}\pm0.20}$  & 89.65$_{\color{gray}\pm0.52}$       \\
GAT        & 93.61$_{\color{gray}\pm0.14}$ & 96.17$_{\color{gray}\pm0.08}$      & 93.87$_{\color{gray}\pm0.11}$  & 90.78$_{\color{gray}\pm0.17}$       \\
GCNII        & 95.03$_{\color{gray}\pm0.59}$ & 96.96$_{\color{gray}\pm0.13}$      & 94.84$_{\color{gray}\pm0.16}$  & 92.35$_{\color{gray}\pm0.09}$       \\
GPRGNN     & 95.13$_{\color{gray}\pm0.09}$ & 96.85$_{\color{gray}\pm0.08}$      & 94.49$_{\color{gray}\pm0.14}$  & 89.32$_{\color{gray}\pm0.29}$       \\
APPNP      & 94.49$_{\color{gray}\pm0.07}$ & 96.54$_{\color{gray}\pm0.07}$      & 94.32$_{\color{gray}\pm0.14}$  & 90.18$_{\color{gray}\pm0.17}$       \\
NAGphormer & 95.75$_{\color{gray}\pm0.09}$ & 97.34$_{\color{gray}\pm0.03}$      & 95.49$_{\color{gray}\pm0.11}$  & 91.22$_{\color{gray}\pm0.14}$       \\
SAN        & 94.51$_{\color{gray}\pm0.15}$ & OOM             & 94.86$_{\color{gray}\pm0.10}$  & 89.83$_{\color{gray}\pm0.16}$       \\
Graphormer & OOM        & OOM             & 92.74$_{\color{gray}\pm0.14}$  & OOM              \\
GraphGPS   & 93.93$_{\color{gray}\pm0.12}$ & OOM             & 95.06$_{\color{gray}\pm0.13}$  & OOM              \\
Exphormer  & 95.77$_{\color{gray}\pm0.15}$ & 97.16$_{\color{gray}\pm0.13}$      & 95.27$_{\color{gray}\pm0.42}$  & 91.59$_{\color{gray}\pm0.31}$       \\
$\mathbf{N}^2$     & 94.44$_{\color{gray}\pm0.45}$ & 97.56$_{\color{gray}\pm0.28}$      & 95.75$_{\color{gray}\pm0.34}$  & 92.51$_{\color{gray}\pm0.13}$       \\ 
\hline
\textbf{AdaR (Ours)}      & \textbf{95.93$_{\color{gray}\pm0.23}$} & \textbf{97.89$_{\color{gray}\pm0.11}$}      & \textbf{95.85$_{\color{gray}\pm0.25}$}  & \textbf{93.38$_{\color{gray}\pm0.17}$}       \\
\hline
\end{tabular}
\end{small}
\end{sc}
% \end{table*}
}
    }
\end{table*}
\begin{figure*}[htb]
    \centering
    \subfigure[w/o. $\mathbf{S}$]{\label{fig:step-info-step}
      \begin{minipage}[t]{0.3\linewidth}
          \centering
          \includegraphics[width=\linewidth]{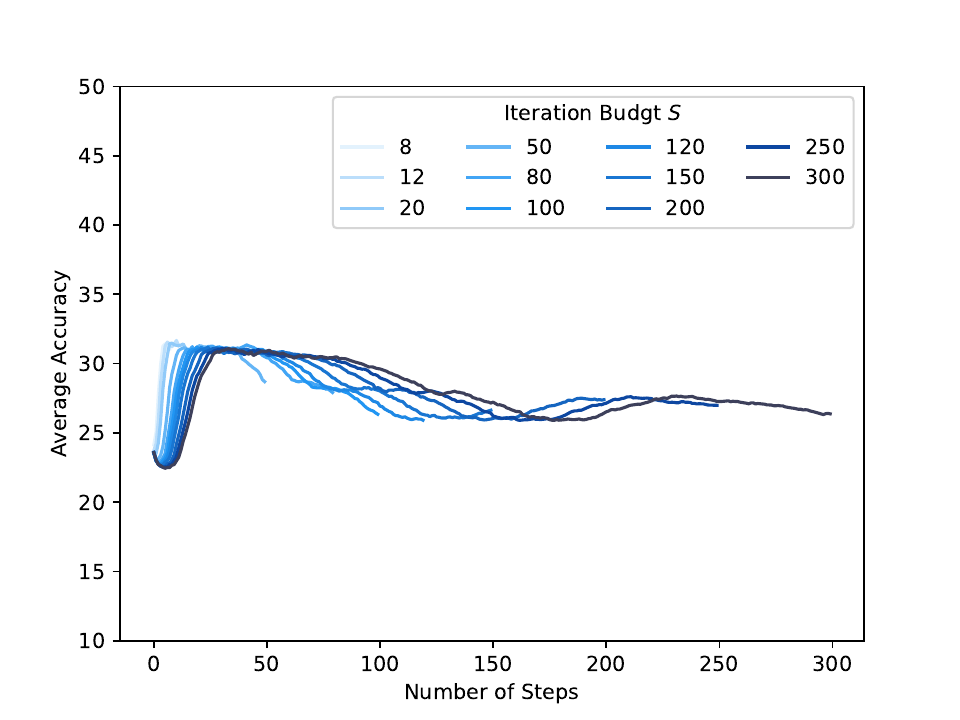}
      \end{minipage}}%
    \hfill
    \subfigure[$\mathbf{S}$ ($\Delta t=1$)]{\label{fig:step-info-absstep}
          \begin{minipage}[t]{0.3\linewidth}
          \centering
          \includegraphics[width=\linewidth]{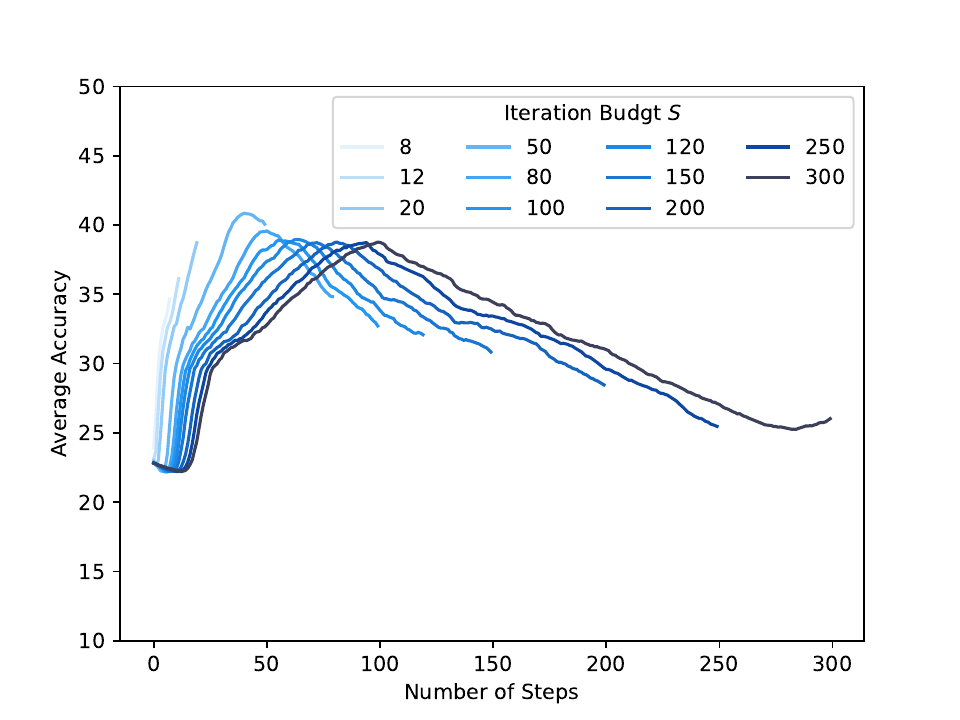}
          \end{minipage}}%
    \hfill
    \subfigure[w/o. $\mathbf{R}$]{\label{fig:step-info-rel}
          \begin{minipage}[t]{0.3\linewidth}
          \centering
          \includegraphics[width=\textwidth]{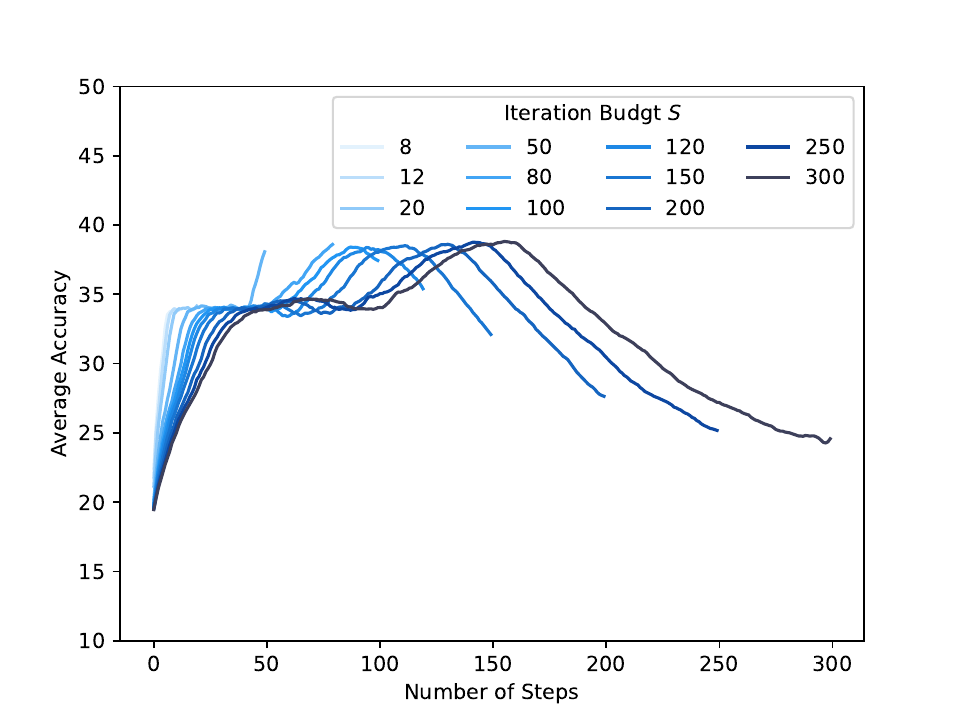}
          \end{minipage}}
    \subfigure[\texttt{AdaR}]{\label{fig:step-adar}
          \begin{minipage}[t]{0.3\linewidth}
          \centering
          \vspace{-0.1in}
          \includegraphics[width=\linewidth]{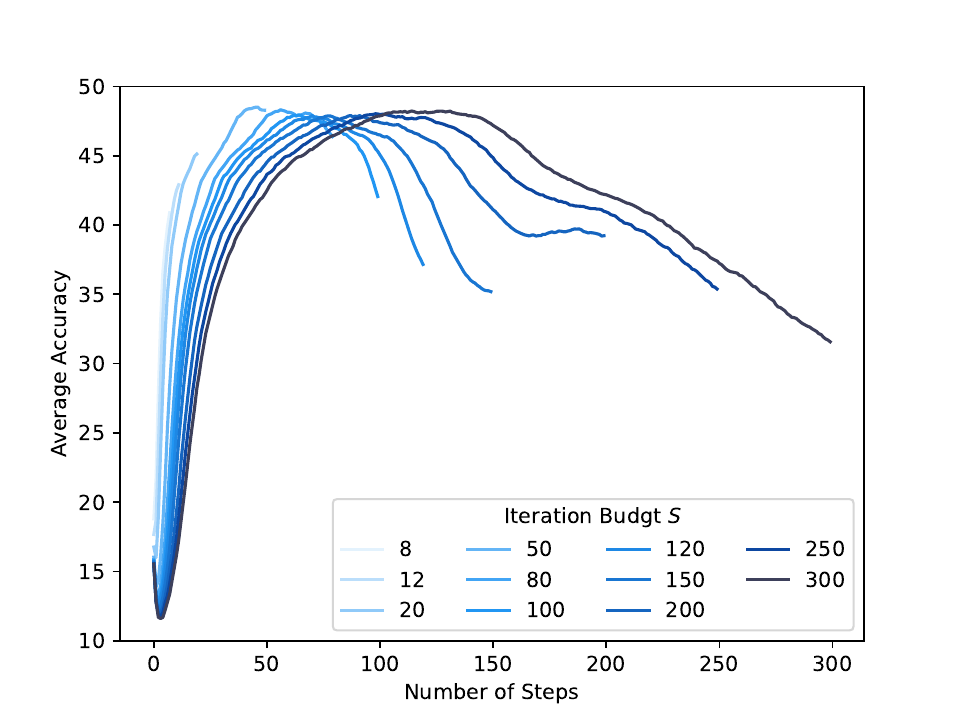}
          \end{minipage}}%
    \hfill
    \subfigure[w/o. $\mathcal{L}_\texttt{full}$]{\label{fig:step-loss-full}
          \begin{minipage}[t]{0.3\linewidth}
          \centering
          \vspace{-0.1in}
          \includegraphics[width=\linewidth]{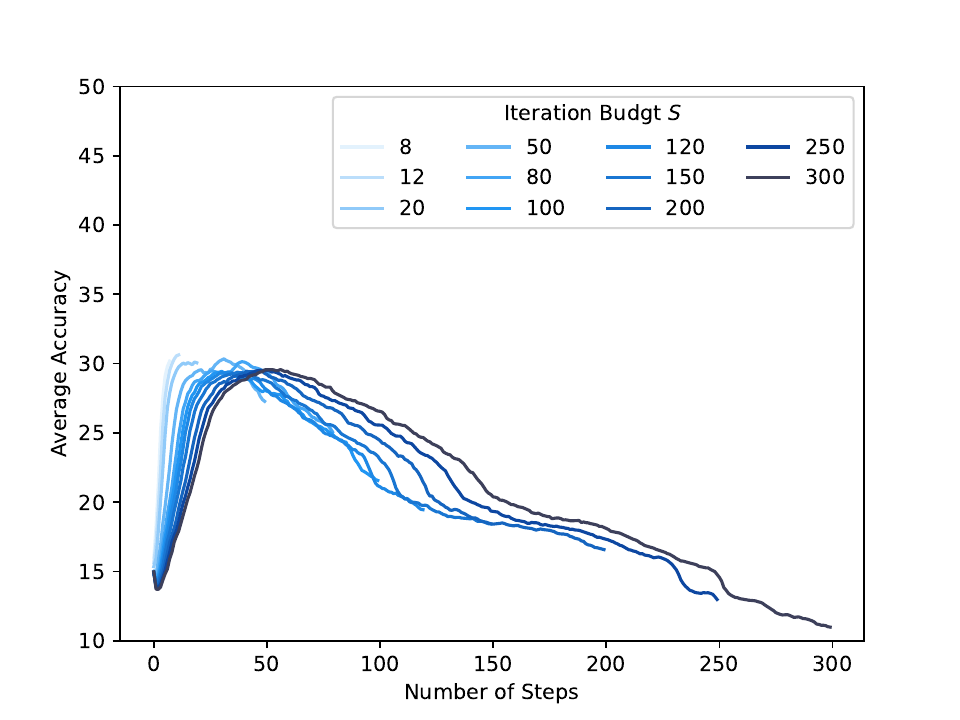}
          \end{minipage}}%
    \hfill
    \subfigure[w/o. $\mathcal{L}_\texttt{step}$]{\label{fig:step-loss-step}
          \begin{minipage}[t]{0.3\linewidth}
          \centering
          \vspace{-0.1in}
          \includegraphics[width=\linewidth]{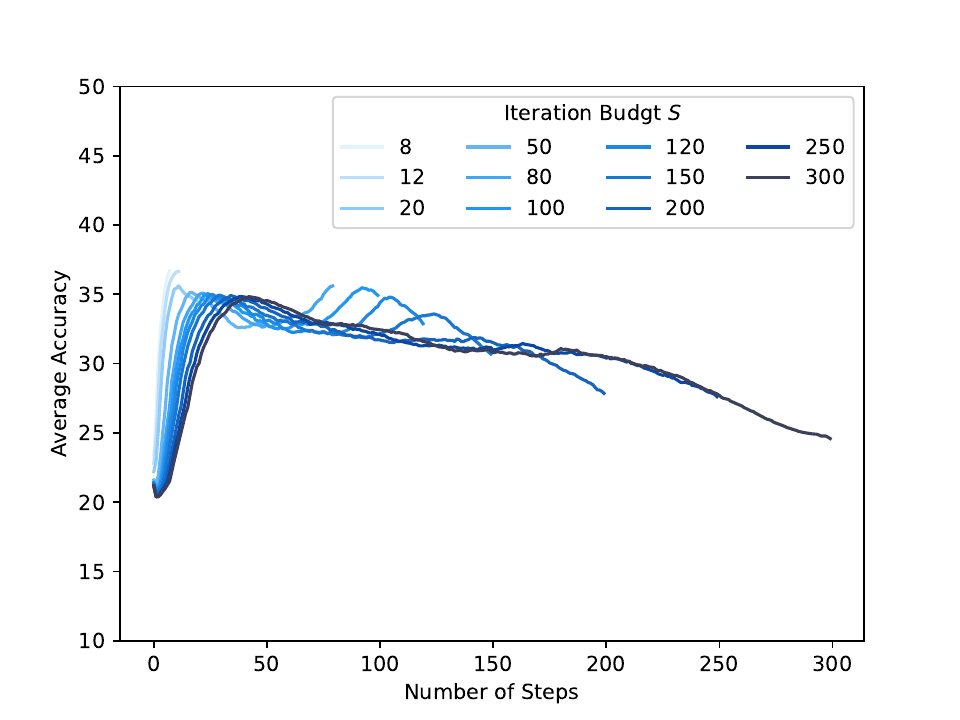}
          \end{minipage}}
    % \vspace{-0.1in}
    \caption{\textbf{Test-time Computing under Different Design Choices.} (a)-(c) studies the effectiveness of different encoded information, where w/o. $\mathbf{S}$ denotes no access to step information, $\mathbf{S} (\Delta t=1)$ denotes encoding step information with a fixed unit step, w/o. $\mathbf{R}$ denotes not employing relations with targets. (d) presents the results of our AdaR. (e)-(f) studies the effectiveness of the supervision signals, where w/o. $\mathcal{L}_\texttt{full}$/$\mathcal{L}_\texttt{step}$ denotes not employing the corresponding optimization target.}
    \label{fig:step}
\end{figure*}
Four types of baseline models are incoporated, including GNN backbones developed for the supervised setting, self-supervised methods, an LLM, and graph foundation models.  
Please refer to Appendix~\ref{ssec:app-setup} for more details. 
The best results among different backbones are reported for each self-supervised methods (Full results at Appendix~\ref{sssec:full-zero-shot}). 
The inference iteration steps are selected based on the node-target relation, with details reported in Section~\ref{ssec:analysis}.

\paragraph{Performance.}
The comparison results between AdaR and baseline models are presented in Tab.~\ref{tab:main-zero}. We can see that AdaR consistently achieves better performance across downstream tasks, which are collected from different domains and encompass both heterophilic and homophilic graphs. Compared to traditional GNN models, which do not follow the recurrent paradigm, AdaR achieves an average absolute improvement of 15.89\%. Compared to $\mathbf{N}^2$, which does not depend on step information and representation-target relations, AdaR achieves an average absolute improvement of 13.52\%. In comparison with zero-shot graph foundation models, our flexible architecture empowers AdaR to employ a large iteration budget, revising the representation delicately without changing model parameters. As a result, AdaR achieves an average improvement of 21.99\% over the compared graph foundation models. AdaR also requires fewer or comparable parameters to other baselines. All these results demonstrate the effectiveness of AdaR in tackling zero-shot inductive tasks, where enabling adaptive recurrent process and adjusting the number of iteration steps gives rise to better adaptability.

\subsection{Supervised Transductive Transfer}
\paragraph{Experimental Setup.} Both AdaR and baselines are trained and evaluated on each dataset separately, including heterophilic~\cite{platonov_CriticalLookEvaluation_2023} and homophilic graphs~\cite{shchur_PitfallsGraphNeural_2019}. 
% Statistics of these datasets are summarized in Tab.~\ref{tab:statistics}.
Please refer to Appendix~\ref{ssec:app-setup} for more details. The best recurrent step of AdaR is selected based on the validation split.

\paragraph{Performance.} Tab.~\ref{tab:main-trans} shows that AdaR with flexible architecture can also surpass baselines under the supervised setting, achieving consistent superior performance on both heterophilic and homophilic graphs. This indicates that even on the same graph, the nodes may require different sizes of receptive fields. Moreover, among AdaR, $\mathbf{N}^2$, and graph transformers that adopt the global message-passing process, AdaR achieves an average absolute improvement of 1.33\% than other models. This demonstrates that even with a global receptive field, the model still requires a flexible architecture to adjust the number of layers adaptively.

\subsection{Model Analysis}\label{ssec:analysis}

\paragraph{Effectiveness of the Encoded Information.}
Based on our theoretical analysis, a recurrent model should explicitly depend on step information $\mathbf{S}$ at each iteration. To enable flexible inference under varying iteration budgets, AdaR adopts \emph{relative} step encoding by fixing the total time $T=1$, rather than using \emph{absolute} time encoding that fixes the unit time $\Delta t=1$. Moreover, to adapt to diverse inputs and targets, AdaR also incorporates the relations between the current node and target representations, denoted as $\mathbf{R}$. We conduct an ablation study to examine the effect of these design choices, as shown in Fig.~\ref{fig:step-info-step}-\ref{fig:step-info-rel}, and compare them with the full AdaR model in Fig.~\ref{fig:step-adar}. Models are pre-trained on amazon-ratings, PubMed, arXiv, and BookHistory with an iteration budget of $S=8$, and evaluated on arxiv-year, WikiCS, SportsFit, Cora, CiteSeer, and DBLP using iteration budgets in $\{8,12,20,50,80,100,120,150,200,250,300\}$. Results are averaged across all evaluation datasets. As shown in Fig.~\ref{fig:step-info-step}, removing step information $\mathbf{S}$ leads to the most severe performance degradation, as the model fails to approximate the target trajectory based on Theorem~\ref{thrm:non-uni-conv}. Furthermore, compared to fixing the total time in AdaR, fixing the unit time $\Delta t=1$ fails to generalize to unseen iteration budgets, where increasing the number of iterations results in deteriorating performance (Fig.~\ref{fig:step-info-absstep}). Regarding the usage of relation information, AdaR, which encodes relations between nodes and targets, achieves better performance compared to the variant with no access to relations (Fig.~\ref{fig:step-info-rel}). These results collectively demonstrate the necessity of incorporating both normalized step information and target-aware relations for flexible recurrent inference.

\paragraph{Effectiveness of the Supervision Signals.}
To provide supervision at each recurrent step, AdaR introduces the gradient of the task loss with respect to the intermediate representations $\mathbf{H}^{(t_s)}$ as supervision signals, and matches it with the predicted updates of $\mathbf{H}^{(t_s)}$. Specifically, $\mathcal{L}_\texttt{full}$ enforces the overall recurrent update to align with the gradient computed at the initial representation, while $\mathcal{L}_\texttt{step}$ supervises the update at each individual step. We conduct ablation studies on these two optimization objectives in Fig.~\ref{fig:step-loss-full} and Fig.~\ref{fig:step-loss-step}, following the same experimental setup as the step information ablation. As shown in the results, removing either $\mathcal{L}_\texttt{full}$ or $\mathcal{L}_\texttt{step}$ leads to performance degradation. Notably, removing $\mathcal{L}_\texttt{full}$ results in a more severe drop in performance (Fig.~\ref{fig:step-loss-full}). This is because $\mathcal{L}_\texttt{full}$ enforces global consistency of the recurrent updates across the entire process, whereas relying solely on $\mathcal{L}_\texttt{step}$ may cause local approximation errors to accumulate over iterations, ultimately harming the final representations. Regarding the ablation on $\mathcal{L}_\texttt{step}$, losing supervision at each recurrent step results in a more fluctuating inference process, where the curves in Fig.~\ref{fig:step-loss-step} exhibit multiple local maxima, indicating oscillatory updates across recurrent steps.

\paragraph{Iteration Exiting Strategy}\label{sssec:exit}
During empirical analysis, we observe that the model performance exhibits a unimodal pattern during the recurrent process, reaching a single maximum at an intermediate step (Fig.~\ref{fig:step-adar}). This suggests that the predicted label distributions of nodes become progressively more aligned with the ground-truth distribution at early iterations, while later iterations lead to degraded or over-smoothed predictions. 
To characterize this behavior from a representation perspective, we analyze the average inner-product relation between node and target representations $\sum_{i,j}(\textbf{H}^{(t_s)}\mathbf{C}^{(t_s)\top})_{i,j}$ at each recurrent step. Results in Fig.~\ref{fig:exit} show that the average value of the inner-product matrix also follows a unimodal pattern, whose maximum closely coincides with the peak model performance. This observation indicates that the average inner-product can serve as an effective surrogate for measuring the alignment between node representations and target representations, and thus provides a practical criterion for iteration exiting.

\begin{figure}[htb]
    \centering
    \subfigure[arxiv-year]{\label{fig:exit-arxiv}
      \begin{minipage}[t]{0.48\linewidth}
          \centering
          \includegraphics[width=\linewidth]{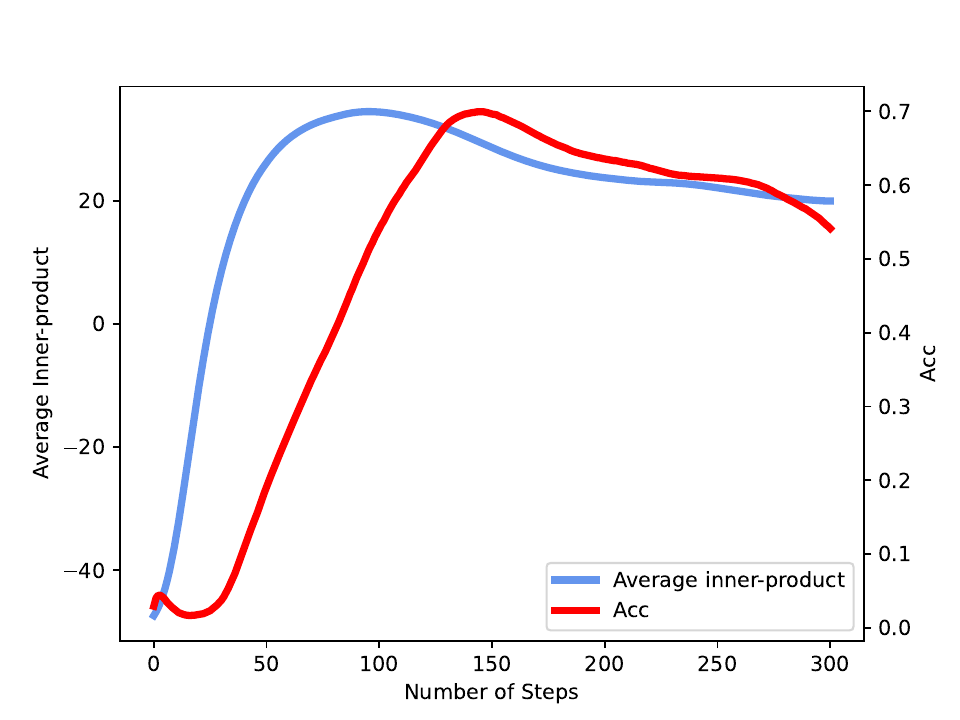}
      \end{minipage}}%
    \hfill
    \subfigure[CiteSeer]{\label{fig:exit-citeseer}
          \begin{minipage}[t]{0.48\linewidth}
          \centering
          \includegraphics[width=\linewidth]{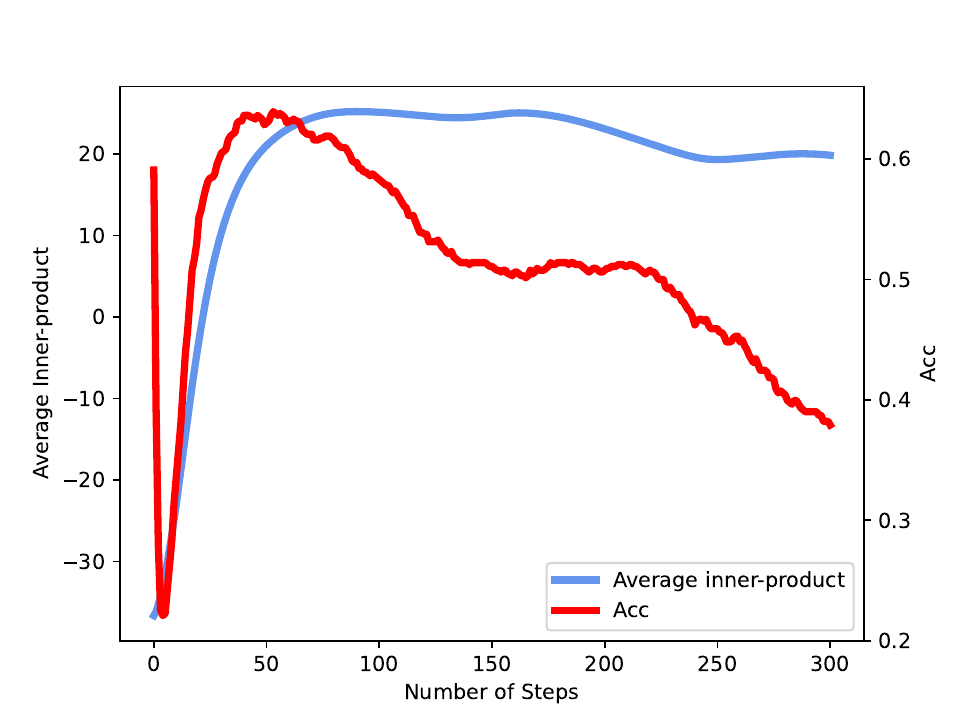}
          \end{minipage}}%
    % \hfill
    % \subfigure[DBLP]{\label{fig:exit-dblp}
    %       \begin{minipage}[t]{0.42\linewidth}
    %       \centering
    %       \includegraphics[width=\textwidth]{fig/exit-DBLP.pdf}
    %       \end{minipage}}
    % \vspace{-0.1in}
    \caption{\textbf{Iteration Exiting Strategy.}}
    \label{fig:exit}
\end{figure}
\begin{figure}[htb]
\centering
\includegraphics[width=0.8\linewidth]{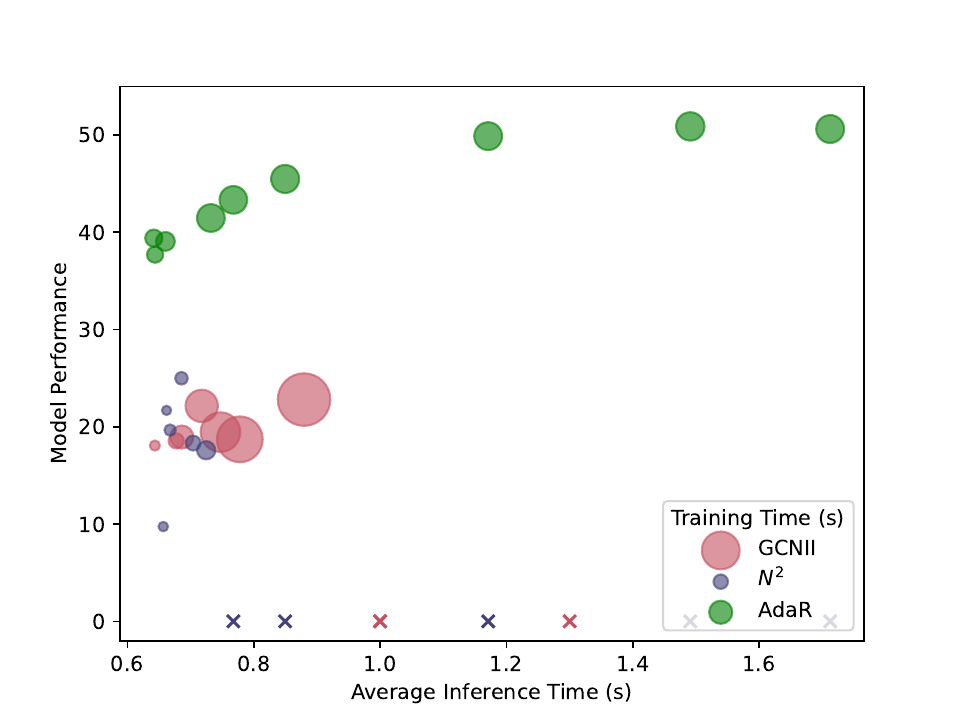}
% \vspace{-0.1in}
\caption{\textbf{Accurcay-Computing Curve.} Bubble size denotes training time. `x' denotes out-of-memory configurations.}\label{fig:acc-compute}
% \vspace{-0.1in}
\end{figure}
\paragraph{Effectiveness under Constrained Resources.}
Different downstream tasks require various optimal receptive fields. To satisfy such requirements, graph nodes must be able to reach nodes at any distance via message passing. Previous works, such as GCNII~\cite{chen_SimpleDeepGraph_2020}, directly achieve this by implementing extremely deep graph models via residual mechanisms. However, this paradigm requires the same architectural configuration during training and inference, leading to longer training time and the out-of-memory issue. In contrast, by employing normalized step information, AdaR can be pre-trained with a small number of iteration budget and make inferences with significantly larger budgets. This empowers AdaR to achieve better performance with less training consumption. In practice, an ideal training budget is around $8$, while the inference budget is around $300$. To evaluate the effectiveness of our recurrent paradigm under constrained training resources, AdaR is compared with GCNII and $\mathbf{N}^2$ regarding model performance and computing efficiency. Grid search is performed on GCNII and $\mathbf{N}^2$ with the number of layers/iterations in $\{8, 20, 50, 100, 150, 200, 250, 300\}$.
% , and dropout probability in $\{0, 0.1, 0.3, 0.5\}$. 
Results in Fig.~\ref{fig:acc-compute} are averaged across different evaluation datasets, where AdaR demonstrates better performance with similar inference computing time and much less training time.

\begin{table}
\caption{\textbf{Cost Comparison.}}
\label{tab:cost}
\centering
\begin{sc}
\begin{small}
\resizebox{0.85\linewidth}{!}{
\setlength{\tabcolsep}{3pt}
\begin{tabular}{clrrrr} 
\hline
    &       & AdaR      & $\mathbf{N}^2$       & LLaGA     & UniGTE     \\ 
\hline
\multirow{2}{*}{\begin{tabular}[c]{@{}c@{}}Time~\\(s)\end{tabular}}  & Train ($\times 10^3$) & 1.59 & 0.54   & 80.11 & 260.44  \\
    & Test  & 4.59      & 0.74     & 2,170.34  & 2093.72    \\
\multirow{2}{*}{\begin{tabular}[c]{@{}c@{}}Mem.~\\(GB)\end{tabular}} & Train & 9.87 & 5.94 & 29.82 & 30.74  \\
                & Test  & 0.57    & 0.49   & 14.04 & 27.25  \\
\hline
\end{tabular}}
% \vspace{-0.1in}
\end{small}
\end{sc}
\end{table}
\paragraph{Cost Analysis.}
To demonstrate the absolute computing efficiency, AdaR is compared with $\mathbf{N}^2$ and LLM-based methods LLaGA and UniGTE. The number of steps for $\mathbf{N}^2$ is fixed as $8$ during both training and inference. Compared to other foundation models, AdaR eliminates the dependence on LLMs to achieve adaptation to various downstream datasets, and therefore demands less memory overhead (Tab.~\ref{tab:cost}). Notably, LLM-based methods transform the input graphs into subgraph sequences and process sequence mini-batches sequentially. As a result, these methods require significantly longer training and test time. This demonstrates the efficiency of AdaR in achieving promising zero-shot transfer performance. Please refer to Appendix~\ref{sssec:app-time} for further time complexity analysis.

\section{Conclusion}
In this paper, we investigated how to reconcile the mismatch between the various receptive field requirements of graph data and the rigid design of pre-defined architectures. To address this problem, we proposed an adaptive recurrent graph model, termed AdaR, which provides a flexible architecture for adaptive test-time computing on graphs, tackling varying receptive field requirements without altering the model architecture. AdaR explicitly encodes normalized step information and representation–target relations into the recurrent updates, enabling adaptive recurrent process on different datasets. Gradient-based supervision signals are designed to guide representation updates throughout the recurrence in AdaR. Empirical results demonstrated that AdaR can be trained with a small iteration budget while being evaluated with substantially larger budgets, consistently outperforming baselines on both inductive and transductive tasks. Limitation discussion is provided in Appendix~\ref{sec:app-limitation}.

\section*{Acknowledgement}
This work was supported in part by the National Key R\&D Program of China under Grant 2023YFC2508704,  in part by the National Natural Science Foundation of China under grant number 62236008, in part by the Natural Science Foundation of Beijing under grant number L251082, and in part by Shandong Provincial Natural Science Foundation under project ZR2025ZD01. The authors would like to thank the anonymous reviewers for their helpful comments and suggestions that improved this manuscript.

% If a paper is accepted, the final camera-ready version can (and usually should) include acknowledgements.  Such acknowledgements should be placed at the end of the section, in an unnumbered section that does not count towards the paper page limit. Typically, this will include thanks to reviewers who gave useful comments, to colleagues who contributed to the ideas, and to funding agencies and corporate sponsors that provided financial support.

\section*{Impact Statement}
This paper presents work whose goal is to advance the field of Machine Learning. There are many potential societal consequences of our work, none of which we feel must be specifically highlighted here.

\bibliography{example_paper}
\bibliographystyle{icml2026}

%%%%%%%%%%%%%%%%%%%%%%%%%%%%%%%%%%%%%%%%%%%%%%%%%%%%%%%%%%%%%%%%%%%%%%%%%%%%%%%
%%%%%%%%%%%%%%%%%%%%%%%%%%%%%%%%%%%%%%%%%%%%%%%%%%%%%%%%%%%%%%%%%%%%%%%%%%%%%%%
% APPENDIX
%%%%%%%%%%%%%%%%%%%%%%%%%%%%%%%%%%%%%%%%%%%%%%%%%%%%%%%%%%%%%%%%%%%%%%%%%%%%%%%
%%%%%%%%%%%%%%%%%%%%%%%%%%%%%%%%%%%%%%%%%%%%%%%%%%%%%%%%%%%%%%%%%%%%%%%%%%%%%%%
\newpage
\appendix
\onecolumn
\renewcommand{\thefigure}{A\arabic{figure}}
\renewcommand{\thetable}{A\arabic{table}}
\renewcommand{\theequation}{A\arabic{equation}}
\section{Proof}\label{sec:app-proof}
% Lemma 1, 无step，无法满足一致收敛
% Assumption, 有step，假设误差上界，可通过UAT保证
% Lemma 2, 无step，无法满足假设
% Theorem 1, 有step，可满足一致收敛性质
\subsection{Lemma~\ref{thrm:non-uni-conv}}
\textbf{Lemma~\ref{thrm:non-uni-conv}} (Non-uniform Convergence without Step Information)
{\it Let $\mathbf{f}_\Theta(\mathbf{H}^{(t_{s-1})},\mathcal{C})$ be a recurrent layer that generates a discrete trajectory $\{{\mathbf{H}^{(t_s)}}\}^{S}_{s=0}$ without the step information $t_s=s\Delta t$. There exist a target trajectory $\{\mathbf{H}^{^{*}(t_s)}\}^{S}_{s=0}$ generated by a continuously differentiable Lipschitz function and a constant $\varepsilon>0$ such that, for sufficiently large $S$, $\sup_{0\le s\le S}\|\mathbf{H}^{(t_s)}-\mathbf{H}^{^{*}{(t_s)}}\|_\mathtt{F}\ge\varepsilon$.}

\begin{proof}
Let the target trajectory $\{{\mathbf{H}^{^*(t_s)}}\}^{S}_{s=0}$ be governed by the ordinary differential equation $\frac{d}{dt_s}\mathbf{H}^{^{*}{(t_s)}}=\mathtt{v}(\mathbf{H}^{^{*}{(t_s)}},\mathcal{C},t_s)$. Applying a first-order Taylor expansion of $\mathbf{H}^{^*(t_s)}$ at $t_0$, we obtain
\begin{equation*}
\mathbf{H}^{^{*}{(t_s)}}=\mathbf{H}^{^{*}{(t_0)}} + \mathtt{v}(\mathbf{H}^{^{*}{(t_0)}},\mathcal{C},t_0)(t_s-t_0)+O(t_s^2).
\end{equation*}
In particular, for $s=1$, $\mathbf{H}^{^{*}{(t_1)}}=\mathbf{H}^{^{*}{(t_0)}} + \mathtt{v}(\mathbf{H}^{^{*}{(t_0)}},\mathcal{C},t_0)\Delta t+O(\Delta t^2)$.

Let the recurrent layer produce $\mathbf{H}^{(t_1)}=\mathbf{H}^{(t_0)}+\delta$, where $\delta$ is independent of $\Delta t$ since the layer does not observe the step information $t_s=s\Delta t$. Assume, for contradiction, that the recurrent layer achieves uniform convergence without the step information.  Then there exists a constant $C>0$, such that for any target trajectory  $\{{\mathbf{H}^{^*(t_s)}}\}^{S}_{s=0}$, 
\begin{equation*}
\|\mathbf{H}^{(t_s)}-\mathbf{H}^{^{*}{(t_s)}}\|_\mathtt{F}\le C\Delta t,\quad s\in\{0,\cdots,S\}. 
\end{equation*}
For $s=1$, this implies
\begin{equation*}
\begin{aligned}
C\Delta t
&\ge\|\mathbf{H}^{(t_1)}-\mathbf{H}^{^{*}{(t_1)}}\|\\
&=\|\delta-\mathtt{v}(\mathbf{H}^{^{*}{(t_0)}},\mathcal{C},t_0)\Delta t-O(\Delta t^2)\|\\
C\Delta t+\|O(\Delta t^2)\|&\ge\|\delta-\mathtt{v}(\mathbf{H}^{^{*}{(t_0)}},\mathcal{C},t_0)\Delta t-O(\Delta t^2)\|+\|O(\Delta t^2)\|.
\end{aligned}
\end{equation*}
By the triangle inequality, there exists a constant $K_1>0$ such that
\begin{equation*}
\begin{aligned}
\|\delta-\mathtt{v}(\mathbf{H}^{^{*}{(t_0)}},\mathcal{C},t_0)\Delta t\| &\le C\Delta t+K_1\Delta t^2\\
\|\frac{\delta}{\Delta t}-\mathtt{v}(\mathbf{H}^{^{*}{(t_0)}},\mathcal{C},t_0)\|&\le C+K_1\Delta t.
\end{aligned}
\end{equation*}
Let $S\rightarrow\infty$, so that $\Delta t=\frac{T}{S}\rightarrow 0$. Since $\delta$ is independent of $\Delta t$, we have $\frac{\delta}{\Delta t}\rightarrow\infty$ unless $\delta=0$ and $\mathbf{H}^{(t_1)}=\mathbf{H}^{(t_0)}$. 
Hence, a step-independent recurrent layer can achieve uniform convergence only for target trajectories with $\|\mathtt{v}(\mathbf{H}^{^{*}{(t_0)}},\mathcal{C},t_0)\|=0$. This contradicts the assumption. Therefore, for any target trajectory satisfying $\|\mathtt{v}(\mathbf{H}^{^{*}{(t_0)}},\mathcal{C},t_0)\|\ne0$ , there exists $\varepsilon>0$ such that, for sufficiently large $S$, $\sup_{0\le s\le S}\|\mathbf{H}^{(t_s)}-\mathbf{H}^{^{*}{(t_s)}}\|_\mathtt{F}\ge\varepsilon$.
\end{proof}

\subsection{Theorem~\ref{thrm:uni-conv}}
\begin{lemma}[Necessity of Step Information]\label{thrm:necessity}
Let $\mathcal{H}\subset\mathbb{R}^m$ and $\mathcal{A}\subset\mathbb{R}^{n\times n}$ be compact and $\mathtt{v}(\mathbf{H},\mathcal{C},t):\mathcal{H}\times\mathcal{A}\times[0,T]\to\mathbb{R}^d$ be continuous. There exists a family of functions $\{\mathtt{v}_\Theta\}$ that uniformly approximates $\mathtt{v}$ on $\mathcal{H}\times\mathcal{A}\times[0,T]$, if and only if the layer input includes a variable that uniquely determines the step information $t$.
\end{lemma}

\begin{proof}~

\textbf{Sufficiency.}~
By the universal approximation theorem, for any $\varepsilon>0$, there exists a parameterized function $\mathtt{v}_\Theta:\mathcal{H}\times\mathcal{A}\times[0,T]\to\mathbb{R}^d$ ({\it e.g.}, a multi-layer perceptron), such that
\begin{equation}
\sup_{(\mathbf{H},\mathcal{C},t)\in\mathcal{H}\times\mathcal{A}\times[0,T]}
\|\mathtt{v}(\mathbf{H},\mathcal{C},t)-\mathtt{v}_\Theta(\mathbf{H},\mathcal{C},t)\|<\varepsilon.
\end{equation}
This establishes sufficiency.

\textbf{Necessity.}~
Since $v:\mathcal{H}\times\mathcal{A}\times[0,T]\to\mathbb{R}^d$ is non-autonomous, there exist $t_1\ne t_2$, $\mathbf{H}$, $\mathcal{C}$, such that $\mathtt{v}(\mathbf{H},\mathcal{C},t_1)\neq \mathtt{v}(\mathbf{H},\mathcal{C},t_2)$. Then the necessity of the lemma can be formulated as that there does not exist a function $\mathtt{g}:\mathcal{H}\times\mathcal{A}\to\mathbb{R}^d$ can uniformly approximate $\mathtt{v}$ on $\mathcal{H}\times\mathcal{A}\times[0,T]$.

Assume, for contradiction, that for arbitrarily small $\varepsilon>0$, there exists a function $\mathtt{g}$ satisfying $\sup_{\mathbf{H},\mathcal{C},t}\|\mathtt{v}(\mathbf{H},\mathcal{C},t)-\mathtt{g}(\mathbf{H},\mathcal{C})\|\le\varepsilon$. In particular, for input $\mathbf{H}$, $\mathcal{C}$, and $t_1\ne t_2$, we have
\begin{equation}
\|\mathtt{v}(\mathbf{H},\mathcal{C},t_1)-\mathtt{g}(\mathbf{H},\mathcal{C})\|\le\varepsilon,\quad
\|\mathtt{v}(\mathbf{H},\mathcal{C},t_2)-\mathtt{g}(\mathbf{H},\mathcal{C})\|\le\varepsilon.
\end{equation}
By the triangle inequality,
\begin{equation}
\|\mathtt{v}(\mathbf{H},\mathcal{C},t_1)-\mathtt{v}(\mathbf{H},\mathcal{C},t_2)\|\le\|\mathtt{v}(\mathbf{H},\mathcal{C},t_1)-\mathtt{g}(\mathbf{H},\mathcal{C})\|+\|\mathtt{v}(\mathbf{H},\mathcal{C},t_2)-\mathtt{g}(\mathbf{H},\mathcal{C})\|\le2\varepsilon.
\end{equation}
Since $\varepsilon$ can be chosen arbitrarily small, this implies $\mathtt{v}(\mathbf{H},\mathcal{C},t_1)=\mathtt{v}(\mathbf{H},\mathcal{C},t_2)$, which contradicts the non-autonomy. Therefore, uniform approximation is impossible without the step information.
\end{proof}

\textbf{Theorem~\ref{thrm:uni-conv}} (Uniform Convergence with Step Information)
{\it For any target trajectory $\{\mathbf{H}^{^{*}(t_s)}\}^{S}_{s=0}$ generated by a continuously differentiable Lipschitz function, there exists a recurrent layer $\mathbf{H}^{(t_{s})}=\mathbf{f}_\Theta(\mathbf{H}^{(t_{s-1})},\mathcal{C},t_{s-1})$ such that ${\mathbf{H}^{(t_s)}}$ converges uniformly to ${\mathbf{H}^{^{*}(t_s)}}$.}

\begin{proof}
Without loss of generality, a recurrent layer accessible to step information can be written in the residual form $\mathbf{f}_\Theta(\mathbf{H}^{(t_{s-1})},\mathcal{C},t_{s-1})=\mathbf{H}^{(t_{s-1})}+\mathtt{v}_\Theta(\mathbf{H}^{(t_{s-1})},\mathcal{C},t_{s-1})\Delta t$. Based on Lemma~\ref{thrm:necessity}, for any $\varepsilon>0$ there exists a mapping function $\mathtt{v}_\Theta$ such that $\sup_{(\mathbf{H},\mathcal{C},t)\in\mathcal{H}\times\mathcal{A}\times[0,T]}
\|\mathtt{v}(\mathbf{H},\mathcal{C},t)-\mathtt{v}_\Theta(\mathbf{H},\mathcal{C},t)\|<\varepsilon$.

Let the target trajectory $\{\mathbf{H}^{^{*}(t_s)}\}^{S}_{s=0}$ satisfy the Lipschitz condition with respect to $\mathbf{H}^{^{*}(t_s)}$, {\it i.e.}, $\|\mathtt{v}(\mathbf{H}_1^{^{*}(t_s)},\mathcal{C},t)-\mathtt{v}(\mathbf{H}_2^{^{*}(t_s)},\mathcal{C},t)\| \le \mathcal{L}_{\mathbf{H}}\|\mathbf{H}_1^{^{*}(t_s)}-\mathbf{H}_2^{^{*}(t_s)}\|$.
By Taylor expansion, $\mathbf{H}^{^{*}(t_{s+1})}=\mathbf{H}^{^{*}(t_s)}+\mathtt{v}(\mathbf{H}^{^{*}(t_s)},\mathcal{C},t_s)\Delta t+O(\Delta t^2)$.
Let $\mathbf{e}_s$ denote the error $\mathbf{H}^{(t_{s})}-\mathbf{H}^{^*(t_{s})}$. Then
\begin{equation}
\begin{aligned}
\mathbf{e}_{s+1}
&=\mathbf{H}^{(t_{s+1})}-\mathbf{H}^{^*(t_{s+1})}\\
&=\mathbf{H}^{(t_{s})}+\mathtt{v}_\Theta(\mathbf{H}^{(t_{s})},\mathcal{C},t_s)\Delta t-\mathbf{H}^{^{*}(t_s)}-\mathtt{v}(\mathbf{H}^{^{*}(t_s)},\mathcal{C},t_s)\Delta t-O(\Delta t^2)\\
&=\mathbf{e}_s+\Delta t\left[\mathtt{v}_\Theta(\mathbf{H}^{(t_{s})},\mathcal{C},t_s)-\mathtt{v}(\mathbf{H}^{^*(t_{s})},\mathcal{C},t_s)\right]-O(\Delta t^2).
\end{aligned}
\end{equation}
Decomposing the difference term yields
\begin{equation}
\begin{aligned}
\mathbf{e}_{s+1}
&=\mathbf{e}_s+\Delta t\left[\mathtt{v}_\Theta(\mathbf{H}^{(t_{s})},\mathcal{C},t_s)-\mathtt{v}(\mathbf{H}^{^*(t_{s})},\mathcal{C},t_s)\right]-O(\Delta t^2)\\
&=\mathbf{e}_s+\Delta t\left[\underset{\texttt{Universal Approximation Error}}{\underbrace{\mathtt{v}_\Theta(\mathbf{H}^{(t_{s})},\mathcal{C},t_s)-\mathtt{v}(\mathbf{H}^{(t_{s})},\mathcal{C},t_s)}}+\underset{\texttt{Lipschitz Condition}}{\underbrace{\mathtt{v}(\mathbf{H}^{(t_{s})},\mathcal{C},t_s)-\mathtt{v}(\mathbf{H}^{^*(t_{s})},\mathcal{C},t_s)}}\right]-O(\Delta t^2)
\end{aligned}
\end{equation}
By uniform approximation, the Lipschitz condition, and the triangle inequality, we obtain
\begin{equation}
\|\mathbf{e}_{s+1}\|
\le(1+\mathcal{L}_\mathbf{H}\Delta t)\|\mathbf{e}_s\|+\varepsilon\Delta t+K_1\Delta t^2,
\end{equation}
for some constant $K_1>0$.

Applying induction and using $\mathbf{e}_0=0$, we have
\begin{equation}
\begin{aligned}
\|\mathbf{e}_{s+1}\|
&\le(1+\mathcal{L}_\mathbf{H}\Delta t)^{s+1}\underset{=0}{\underbrace{\|\mathbf{e}_0\|}}+\sum_{i=0}^{s}(1+\mathcal{L}_\mathbf{H}\Delta t)^{i}(\varepsilon\Delta t+K_1\Delta t^2)\\
&=\frac{(1+\mathcal{L}_\mathbf{H}\Delta t)^{s}-1}{\mathcal{L}_\mathbf{H}\Delta t}(\varepsilon\Delta t+K_1\Delta t^2)\\
&=\frac{\exp[s\ln(1+\mathcal{L}_\mathbf{H}\Delta t)]-1}{\mathcal{L}_\mathbf{H}}(\varepsilon+K_1\Delta t)\\
&\le\frac{\exp[s\mathcal{L}_\mathbf{H}\Delta t]-1}{\mathcal{L}_\mathbf{H}}(\varepsilon+K_1\Delta t)\\
&=\frac{\exp[t_{s}\mathcal{L}_\mathbf{H}]-1}{\mathcal{L}_\mathbf{H}}(\varepsilon+K_1\Delta t).
\end{aligned}
\end{equation}
Since $t_{s}\le T$, it follows that $\sup_s\|e_s\|\le\frac{\exp(T\mathcal{L}_\mathbf{H})-1}{\mathcal{L}_\mathbf{H}}(\varepsilon+K_1\Delta t)$. When $\varepsilon\rightarrow 0$, the error bound decreases to $\frac{\exp(T\mathcal{L}_\mathbf{H})-1}{\mathcal{L}_\mathbf{H}}K_1\Delta t=C\Delta t$, which proves the uniform convergence.
\end{proof}

\section{Pseudo Nodes in Adaptive Graph Models}\label{sec:app-related}
To unify the varying number of targets across different tasks, AdaR adopts pseudo nodes to aggregate information from the target representation and update its own representations accordingly. As a result, these pseudo nodes can be regarded as abstractions of targets, serving as fixed-size proxies for relation computation during the recurrent process. The concept of pseudo nodes has emerged in many adaptive graph models~\cite{liu_OneAllTraining_2023,wang_GFTGraphFoundation_2024,zhao_AllOneOne_2024}. However, AdaR introduces pseudo nodes for a different purpose via a different implementation strategy. Regarding purposes, pseudo nodes in previous works are to readout the learned representations~\cite{liu_OneAllTraining_2023,wang_GFTGraphFoundation_2024} or to enhance the cross-graph communication~\cite{zhao_AllOneOne_2024}, yet AdaR employs pseudo nodes to unify different targets. Regarding implementations, pseudo nodes in AdaR are shared parameters for all graphs. Conversely, pseudo nodes in previous works either cannot be jointly optimized with the overall framework or fail to transfer across tasks~\cite{liu_OneAllTraining_2023,wang_GFTGraphFoundation_2024,zhao_AllOneOne_2024}.

\section{Implementation of $\texttt{MP}$}\label{sec:app-mp}
We implement $\texttt{MP}$ based on the global and local message passing in $\mathbf{N}^2$~\cite{sun_DynamicMessagePassing_2024} with minor modification, including replacing the path integral function with a combination of inner product and $\texttt{softmax}$, replacing the update function with a velocity function in both Eq.~\ref{eq:glob-mp-update} and Eq.~\ref{eq:gnode}.

\subsection{Global Message Passing and Implementation of $\texttt{AGG}$}\label{ssec:app-mp-g}
Given two sets of node elements, such as graph nodes, pseudo nodes, and targets, one set of elements can perform global message passing via the other set of elements. For example, when the number of graph nodes is extremely large, graph nodes can perform global message passing with a small number of pseudo nodes as surrogates to ensure linear complexity. Let $\mathbf{X}_\texttt{in}$ denote the input representation of the global message-passing process, $\mathbf{X}_\texttt{sur}$ denote the surrogate representation, and $\mathbf{X}_\texttt{cond}$ denote a certain condition representation, such as target representations. The global message-passing process at the $s$-th step can be formulated as
\begin{align}
    &\texttt{(Diffuse)}
    &&\mathbf{G} =\psi(\mathbf{X}_\texttt{sur}\mathbf{X}_\texttt{in}^{\top})\mathtt{NL}(\mathbf{X}_\texttt{in}), \label{eq:glob-mp-np}\\
    &\texttt{(Refine)}
    &&\mathbf{\hat{G}} = \psi(\mathbf{X}_\texttt{sur}\mathbf{X}_\texttt{sur}^{\top})\mathbf{G},\label{eq:glob-mp-agg}\\
    &\
    &&\mathbf{X}_\texttt{sur}\leftarrow \mathbf{X}_\texttt{sur} + \mathtt{NL}([\mathbf{\hat{G}}\|\mathbf{1}_{n}\mathbf{S}_{s,\cdot}^\top\|\mathbf{X}_\texttt{sur}\mathbf{X}_\texttt{cond}^{\top})\Delta t,  \label{eq:glob-mp-update}\\
    &\texttt{(Collect)}
    &&\mathbf{M}^\mathtt{glob}\! =\psi(\mathbf{X}_\texttt{in}\mathbf{X}_\texttt{sur}^{\top})\mathtt{NL}(\mathbf{\hat{G}}), \label{eq:glob-mp-pn}
\end{align}
where $\psi$ constructs edge weights for message passing between different node elements, such as graph nodes, pseudo nodes, and targets. In practice, $\psi$ is implemented as a combination of inner product and $\texttt{softmax}$. Eq.~\ref{eq:glob-mp-update} formulates the global feature refinement at the surrogate level, where $\mathtt{NL}([\mathbf{\hat{G}}\|\mathbf{1}_{n}\mathbf{S}_{s,\cdot}^\top\|\mathbf{X}_\texttt{sur}\mathbf{X}_\texttt{cond}^{\top})$ follows the same form in Eq.~\ref{eq:velocity}, considering the current representation, the step information, its relation with certain conditions. Eq.~\ref{eq:glob-mp-pn} formulates the global message-passing outputs, where $\mathbf{M}^\mathtt{glob}$ encodes global information.

This process can be summarized as ``input$\rightarrow$surrogate$\rightarrow$input". For simplicity, we compile Eq.~\ref{eq:glob-mp-np}-\ref{eq:glob-mp-pn} as $\mathbf{M}^{\mathtt{glob}}, \mathbf{X}_\texttt{sur}\leftarrow\mathtt{GlobMP}(\mathbf{X}_\texttt{in}, \mathbf{X}_\texttt{sur},\mathbf{X}_\texttt{cond})$. Eq.~\ref{eq:glob-mp-np} and Eq.~\ref{eq:glob-mp-agg} formluate the aggregation fucntion $\texttt{AGG}(\cdot)$ in Eq.~\ref{eq:proxy-abstract}.

\subsection{Local Message Passing}\label{ssec:app-mp-l}
Topology-coupled message passing is employed to encode the local structures. The resulted local message-passing process for graph node $v$ can be formulated as
\begin{equation}\label{eq:local-mp}
    \mathbf{m}^\mathtt{local}_{v_i} \!=\! \frac{1}{|\mathcal{N}({v_i})|+1}
    \left[ \mathbf{m}_{v_i} \!+\! \!
    \sum_{v'\in\mathcal{N}({v_i})}\mathtt{NL}(\mathbf{m}_{v_j}) \right].
\end{equation}
% where $||$ denotes the concatenation operation, $\mathbf{E}$ denotes the edge features. 
Through local message passing, graph nodes aggregate messages from their adjacent nodes. Eq.~\ref{eq:local-mp} can be compiled as $\mathbf{M}^{\mathtt{local}}=\mathtt{LocalMP}(\mathbf{M}, \mathbf{A})$ for all the graph nodes.

\subsection{Velocity Computing}\label{ssec:app-mp-v}
Based on the formulation from $\mathbf{N}^2$, we now specify the computing process in AdaR. Specifically, given the current representations of graph nodes and targets, AdaR first performs $\texttt{GlobMP}$ to update pseudo-node representations and collect global features. In practice, we also update the target representations recurrently at this stage, depending on their current representation $\mathbf{C}^{(t_{s-1})}$, step information $\mathbf{S}$, and their relation with graph nodes. Therefore, similar to the computation of node-target relations, this process also requires pseudo nodes as surrogates. Let $\mathbf{P}^{(\texttt{p},t_{s})}\in\mathbb{R}^{n_{\texttt{p}}\times d}$ denote the pseudo-node representations, where $\texttt{p}\in\{\texttt{n},\texttt{c}\}$ indicates the associated element type (graph nodes or targets), $n_{\texttt{p}}$ is the number of pseudo nodes, and $\mathbf{P}^{(\texttt{p},0)}$ is a learnable initialization. Given graph node representation $\mathbf{H}^{(t_{s-1})}$, target representation $\mathbf{C}^{(t_{s-1})}$, and pseudo-node representation $\mathbf{P}^{(\texttt{p},t_{s-1})},\texttt{p}\in\{\texttt{c},\texttt{n}\}$ at the $s$-th recursive step, the velocity computing process can be formulated as
\begin{align}
    &\texttt{(Graph-node Abs.)}
    &&\mathbf{M}_\texttt{pn}^{\mathtt{glob}(t_{s})}, \mathbf{P}^{(\texttt{n},t_{s})} \!=\! \mathtt{GlobMP}
    \left(\mathbf{H}^{(t_{s-1})}, \mathbf{P}^{(\texttt{n},t_{s-1})}, \mathbf{P}^{(\texttt{n},t_{s-1})}\right),\label{eq:pnode-graph} \\
    &\texttt{(Target Abs.)}
    &&\texttt{\quad \ -\quad \ }, \mathbf{P}^{(\texttt{c},t_{s})} \!=\! \mathtt{GlobMP}
    \left(\mathbf{C}^{(t_{s-1})}, \mathbf{P}^{(\texttt{c},t_{s-1})}, \mathbf{P}^{(\texttt{c},t_{s-1})}\right),\label{eq:pnode-target} \\
    &\texttt{(Target Rep. Update)}
    &&\mathbf{M}_\texttt{cn}^{\mathtt{glob}(t_{s})}, \mathbf{C}^{(t_{s})} \!=\! \mathtt{GlobMP}
    \left(\mathbf{H}^{(t_{s-1})}, \mathbf{C}^{(t_{s-1})}, \mathbf{P}^{(\texttt{n},t_{s})}\right),\label{eq:cnode} \\
    &\texttt{(Local MP.)}
    &&\mathbf{M}^{\mathtt{local}(t_{s})}=\mathtt{LocalMP}
    \left[
    \left(\mathbf{M}_\texttt{pn}^{\mathtt{glob}(t_{s})} \| \mathbf{M}_\texttt{cn}^{\mathtt{glob}(t_{s})} \| \mathbf{H}^{(t_{s-1})}\right), \mathbf{A}\right], \label{eq:local}\\
    &\texttt{(Node Rep. Update)}
    &&\Delta\mathbf{H}^{(t_{s})}=\texttt{NL}([\mathbf{M}^{\mathtt{local}(t_{s})}\|\mathbf{1}_n\mathbf{S}_{s-1,\cdot}^\top\|\mathbf{H}^{(t_{s-1})}\mathbf{P}^{(\texttt{c},t_{s})\top}]),\label{eq:gnode}
\end{align}
where $\mathbf{M}_\texttt{pn}^{\mathtt{glob}(t_s)}$, $\mathbf{M}_\texttt{cn}^{\mathtt{glob}(t_s)}$ denote the global features of graph nodes, extracted via pseudo nodes and targets, respectively. 
Eq.~\ref{eq:pnode-graph} formulates the graph-node abstraction process with pseudo nodes (\texttt{Graph-node Abs.}). It takes graph-node representation $\mathbf{H}^{(t_{s-1})}$ as inputs $\mathbf{X}_\texttt{in}$, pseudo-node representation $\mathbf{P}^{(\texttt{n},t_{s-1})}$ as the global message-passing surrogate $\mathbf{X}_\texttt{sur}$ and condition $\mathbf{X}_\texttt{cond}$, updating the pseudo-node representations to serve as an abstraction of graph nodes.
Similarly, Eq.~\ref{eq:pnode-target} formulates the target abstraction process with pseudo nodes (\texttt{Target Abs.}). It takes target representation $\mathbf{C}^{(t_{s-1})}$ as inputs $\mathbf{X}_\texttt{sur}$, pseudo-node representation $\mathbf{P}^{(\texttt{c},t_{s-1})}$ as the global message-passing surrogate $\mathbf{X}_\texttt{sur}$ and condition $\mathbf{X}_\texttt{cond}$, updating the pseudo-node representations to serve as an abstraction of targets. Eq.~\ref{eq:cnode} updates the target representations conditioned on the graph node abstraction $\mathbf{P}^{(\texttt{n},t_{s})}$ (\texttt{Target Rep Update}). Eq.~\ref{eq:gnode} follows Eq.~\ref{eq:velocity} to compute the velocity.

Compiling Eq.~\ref{eq:pnode-graph}-\ref{eq:local} gives rise to the message passing function $\texttt{MP}(\cdot)$, while compling Eq.~\ref{eq:pnode-graph}-\ref{eq:gnode} yields the velocity function $\texttt{v}_\Theta(\mathbf{H}^{(t_{s-1})},\mathcal{C},t_{s-1})$.

\begin{figure*}[htb]
    \centering
    \subfigure[Fixed Unit Time Step]{\label{fig:step-unit}
      \begin{minipage}[t]{0.45\linewidth}
          \centering
          \includegraphics[width=\linewidth]{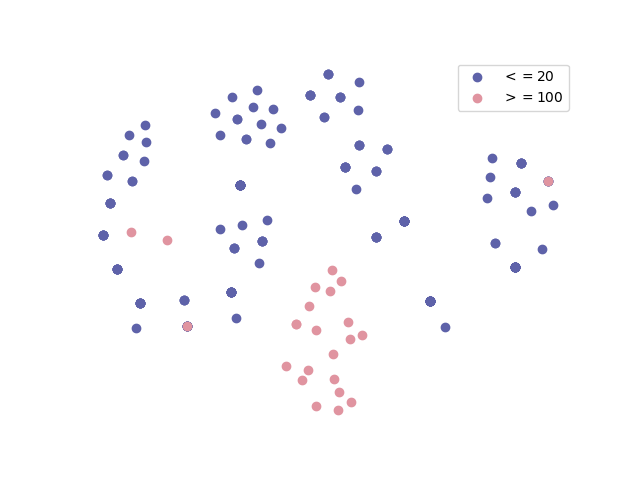}
      \end{minipage}}%
    \hfill
    \subfigure[Fixed Total Time]{\label{fig:step-total}
          \begin{minipage}[t]{0.45\linewidth}
          \centering
          \includegraphics[width=\linewidth]{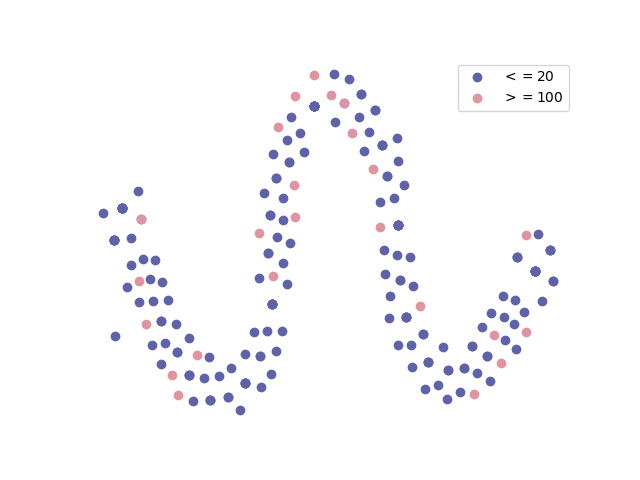}
          \end{minipage}}%
    \caption{\textbf{Comparison with Different Step Information.} `$\le20$' denotes iteration budgets sampled from $[2,20]$, `$\ge100$' denotes iteration budgets sampled from $[100,300]$.}
    \label{fig:step}
\end{figure*}
\section{Analysis on Step Information Encoding}\label{sec:app-relative}
Given the step information $t_s=s\Delta t=sT/S$, two common design choices include fixing the unit time step $\Delta t$ or fixing the total time $T$. To ensure generalizability to different iteration budgets, AdaR fixes the total time ({\it e.g.}, $T=1$). Specifically, different choices of the iteration budget $S$ correspond to different discretizations of the same target trajectory, ensuring that the model is always trained to produce $\mathbf{H}^{(1)}$ as the final representation. To validate this specific implementation choice, we provide further analysis.

\paragraph{Theoretical Analysis.} Based on the implementation of the sinusoidal position encoding function, we have step information embedded as $[\sin(w_k \cdot t_s),\cos(w_k \cdot t_s)]_k$. Let the largest iteration budget during training be $T_0$. Without fixing the total time $T$, inference may involve $T\gg T_0$, leading to phase shifts $w_k\cdot(t_s-T_0)$ proportional to $T$ and large embedding deviations. In contrast, fixing the total time with $T=1$ rescales $t_s$ to $[0,1]$. Therefore, the maximum shift is bounded by the constant $w_k$ independent of $T$, implying smaller embedding shifts.

\paragraph{Empirical Analysis.} To empirically validate our encoding strategy, we select the iteration budget from two distinct ranges, \textit{i.e.} smaller budgets $[2, 20]$ and larger budgets $[100, 300]$. The step indices are embedded with sinusoidal encoding and visualized with t-SNE~\cite{t-SNE}. Results in Fig.~A\ref{fig:step-unit} show that when embedded with a fixed unit time step, embeddings from the smaller and larger budget ranges form clearly separated clusters with a distinct boundary, indicating a significant distribution shift between the two groups. In contrast, Fig.~A\ref{fig:step-total} shows that using a fixed total time yields a more mixed distribution, where embeddings from both ranges lie on a shared low-dimensional manifold, effectively alleviating the distribution shift. These results demonstrate the effectiveness of using a fixed total time rather than a fixed unit time step for step information encoding.

\begin{table*}
\caption{\textbf{Dataset Statistics.}}
\label{tab:statistics}
\centering
\begin{sc}
\begin{small}
\begin{tabular}{llrrr} 
\hline
Dataset         & Task Type                     & \# Nodes & \# Edges  & \# Classes  \\ 
\hline
arXiv           & inductive    & 169,343  & 1,166,243 & 40          \\
BookHistory     & inductive    & 41,551   & 358,574   & 12          \\
PubMed          & inductive     & 19,717   & 44,338    & 3           \\
arxiv-year      & inductive     & 169,343  & 1,166,243 & 5           \\
Cora            & inductive     & 2,708    & 10,556    & 7           \\
CiteSeer        & inductive     & 3,186    & 8,450     & 6           \\
DBLP            & inductive     & 14,376   & 431,326   & 4           \\
SportsFit       & inductive     & 173,055  & 1,773,500 & 13          \\
WikiCS          &inductive     & 11,701   & 431,726   & 10          \\
amazon-ratings  & inductive/transductive        & 24,492   & 93,050    & 5    \\
questions       & transductive & 48,921   & 245,778   & 2           \\
tolokers        & transductive   & 11,758   & 119,043   & 2           \\
minesweeper     & transductive   & 10,000   & 163,788   & 2           \\
AmazonComputer  & transductive   & 13,381   & 495,924   & 10          \\
AmazonPhoto     & transductive   & 7,487    & 153,540   & 8           \\
CoauthorPhysics & transductive   & 18,333   & 519,000   & 15          \\
CoauthorCS      & transductive   & 34,493   & 39,402    & 5           \\
\hline
\end{tabular}
\end{small}
\end{sc}
\end{table*}
\section{Experimental Details}
\subsection{Datasets}
For inductive transfer tasks, we adopt the textual attribute datasets~\cite{chen_TextspaceGraphFoundation_2024}, including (arXiv, bookhistory, amazon-ratings, and PubMed) for training and (arxiv-year, WikiCS, SportsFit, Cora, CiteSeer, and DBLP) for evaluation. arXiv, arxiv-year, PubMed, Cora, CiteSeer, and DBLP are citation graphs. Bookhistory, amazon-ratings, and SportsFit are economic-commercial graphs, while WikiCS is collected from Wikipedia.

For transductive tasks, we adopt both heterophilic datasets (questions, amazon-ratings, tolokers, and minesweeper)~\cite{platonov_CriticalLookEvaluation_2023} and homophilic datasets (CoauthorCS, CoauthorPhysics, AmazonPhoto, and AmazonComputers)~\cite{shchur_PitfallsGraphNeural_2019}. Specifically, CoauthorCS and CoauthorPhysics are derived from the Microsoft Academic Graph and model academic collaboration networks. AmazonComputers and AmazonPhoto are product co-purchase networks collected from Amazon. Questions is built from the Yandex Q\&A platform and represents user interaction graphs collected over a one-year period (September 2021–August 2022). Amazon-ratings is based on the Amazon co-purchasing network provided by the SNAP dataset collection~\cite{snapnets}. Tolokers represents crowdsourcing participation data from the Toloka platform. Minesweeper is a synthetic grid graph of size $100\times 100$, where nodes correspond to grid cells.

\subsection{Experimental Setup}\label{ssec:app-setup}
In both inductive and transductive experiments, we follow the standard splits as in the original papers for all the benchmarks. We adopt Adam \cite{kingma_AdamMethodStochastic_2015} as optimizer and set weight decay as $1\times 10^{-6}$. For inductive experiments, the training epoch is set as $10,000$, the learning rate is set to $1\times 10^{-4}$, the training iteration budget is set as $8$, and the test-time iteration budget is set as $300$. For transductive experiments, the training epoch is set as $1,000$, the learning rate is set to $1\times 10^{-3}$, the training iteration budget is selected from $\{2, 4, 8\}$, the test-time iteration budget is selected from $\{8, 20, 50, 100, 150, 200, 250, 300\}$ and the dropout probability is selected from $\{0, 0.1, 0.2, 0.3\}$. Models are selected based on the average performance on training datasets for the zero-shot setting, while based on validation sets for the transductive setting.

For the comparison under the zero-shot inductive transfer setting, we incorporate four types of baseline models, including supervised GNNs (GCN~\cite{kipf_SemiSupervisedClassificationGraph_2017}, GAT~\cite{velickovic_GraphAttentionNetworks_2018}, GCNII~\cite{chen_SimpleDeepGraph_2020}, GraphSAGE~\cite{hamilton_InductiveRepresentationLearning_2017}, and the recurrent graph model $\mathbf{N}^2$~\cite{sun_DynamicMessagePassing_2024}), self-supervised methods (DGI~\cite{velickovic_DeepGraphInfomax_2018}, GRACE~\cite{zhu_DeepGraphContrastive_2020}), a large language model Vicuna-7B~\cite{zheng_JudgingLLMJudgeMTBench_2023}, and graph foundation models (OFA~\cite{liu_OneAllTraining_2023}, LLaGA~\cite{chen_LLaGALargeLanguage_2024}, GFT~\cite{wang_GFTGraphFoundation_2024}, MDGFM~\cite{wang_MultiDomainGraphFoundation_2025}, RiemannGFM~\cite{sun_RiemannGFMLearningGraph_2025}, and UniGTE~\cite{wang_UniGTEUnifiedGraph_2025}).

For the comparison under the supervised transductive setting, we incorporate both GNNs (GCN~\cite{kipf_SemiSupervisedClassificationGraph_2017}, GAT~\cite{velickovic_GraphAttentionNetworks_2018}, 
APPNP~\cite{gasteiger_PredictThenPropagate_2018}, GCNII~\cite{chen_SimpleDeepGraph_2020}, H$_2$GCN~\cite{zhu_HomophilyGraphNeural_2020}, FAGCN~\cite{bo_LowfrequencyInformationGraph_2021}, GPRGNN~\cite{chien_AdaptiveUniversalGeneralized_2022}, GloGNN~\cite{li_FindingGlobalHomophily_2022}, and $\mathbf{N}^2$~\cite{sun_DynamicMessagePassing_2024}) and graph transformers (Graphormer~\cite{ying_TransformersReallyPerform_2021}, GraphGPS~\cite{rampasek_RecipeGeneralPowerful_2022}, NAGphormer~\cite{chen_NAGphormerTokenizedGraph_2023}, SAN~\cite{kreuzer_RethinkingGraphTransformers_2021}, and Exphormer~\cite{shirzad_ExphormerSparseTransformers_2023}).

Except for the reproducing experiments conducted on NVIDIA A100 for Vicuna-7B~\cite{zheng_JudgingLLMJudgeMTBench_2023}, OFA~\cite{liu_OneAllTraining_2023}, LLaGA~\cite{chen_LLaGALargeLanguage_2024}, RiemannGFM~\cite{sun_RiemannGFMLearningGraph_2025}, UniGTE~\cite{wang_UniGTEUnifiedGraph_2025}, and cost comparison, all the experiments are conducted on a single NVIDIA GeForce RTX 4090.

For the cost comparison, models are pre-trained on arXiv, bookhistory, amazon-ratings, and PubMed, evaluated on WikiCS, SportsFit, and Cora.

\subsection{Additional Results}

\begin{table*}
\caption{\textbf{Full Results for Self-supervised Methods (measured by accuracy: \%).}}
\label{tab:ssl}
\centering
\begin{sc}
\begin{small}
\begin{tabular}{llcccccc} 
\hline
            & Backbones & arxiv-year       & WikiCS     & SportsFit  & Cora       & CiteSeer         & DBLP              \\ 
\hline
\multirow{3}{*}{DGI}   & GAT       & OOM       & 16.14    & 9.13  & 17.88 & OOM       & 5.40    \\
                       & GCN       & OOM       & 14.73    & 15.28 & 35.88 & OOM       & 8.26    \\
                       & GraphSAGE & OOM       & 26.81    & 15.47 & 36.78 & OOM       & 2.82    \\ 
\hline
\multirow{3}{*}{GRACE} & GAT       & OOM       & 16.77    & 19.68 & 23.29 & OOM       & 3.84    \\
                       & GCN       & OOM       & 18.18    & 10.15 & 26.46 & OOM       & 5.06    \\
                       & GraphSAGE & OOM       & 25.39    & 4.59  & 19.86 & OOM       & 7.26    \\
\hline
\textbf{AdaR (Ours)} & & \textbf{36.15}       & \textbf{37.81} & \textbf{31.56} & \textbf{41.28} & \textbf{59.39}       & \textbf{59.83}        \\
\hline
\end{tabular}
\end{small}
\end{sc}
\end{table*}
\subsubsection{Full results for Zero-shot Inductive Transfer}\label{sssec:full-zero-shot}
For baselines, we incorporate four types of baseline models, including supervised GNNs, self-supervised methods, a large language model Vicuna-7B, and zero-shot graph foundation models. 
The self-supervised methods are implemented with GCN, GAT, and GraphSAGE. The best results among the three backbones are reported for each method in the main context. Here, we present the full results in Tab.~\ref{tab:ssl}, where AdaR consistently outperforms self-supervised methods with different backbones.

\begin{figure*}[htb]
\centering
\includegraphics[width=0.45\linewidth]{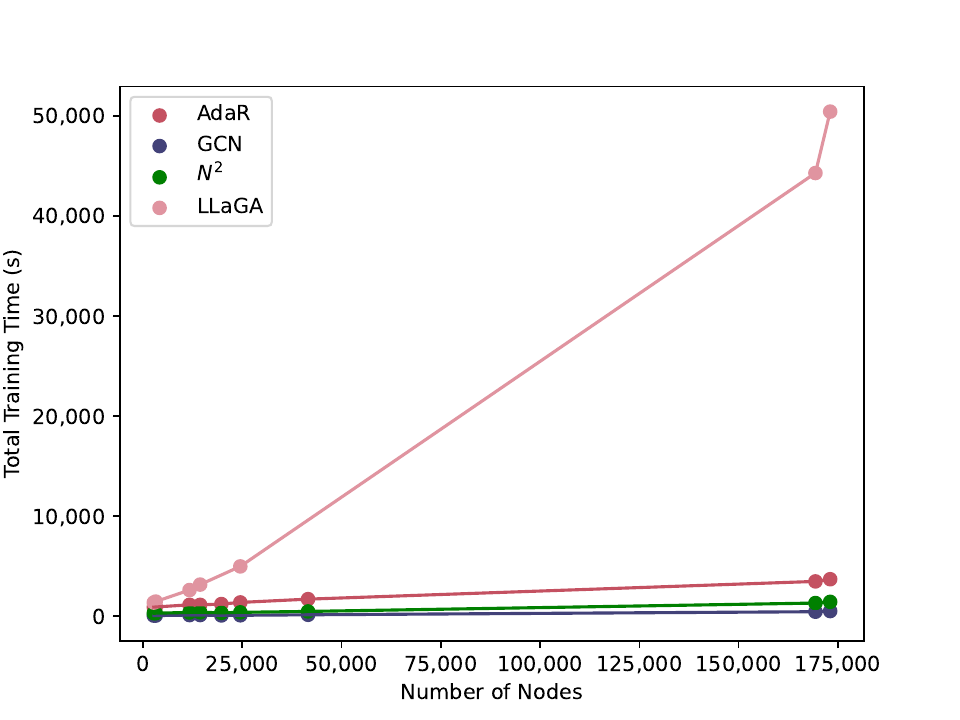}
% \vspace{-0.1in}
\caption{\textbf{Training Time Regarding Different Numbers of Nodes.}}\label{fig:app-time}
\end{figure*}
\subsubsection{Complexity Analysis}\label{sssec:app-time}
Similar to classical GNNs such as GCN and GAT, AdaR requires a computing complexity of $O(n+m)$, where $n$ and $m$ denote the number of nodes and edges, respectively. In contrast, LLM-based models such as LLaGA often require a computing complexity of $\frac{O(k^2)\cdot n}{B}$, where $B$ and $k$ denote the batch size and the number of nodes in the mini-batch, respectively. Due to the quadratic time complexity of LLMs, these models can only operate on mini-batches and compute multiple batches in a sequential manner. Empirically, we conduct complexity analysis on real-world datasets with varying scales. Baselines include classic GNN model GCN, recurrent model $\mathbf{N}^2$, adaptive recurrent model AdaR, and LLM-based model LLaGA. Results in Fig.~\ref{fig:app-time} show the linear time complexity of AdaR, comparable to other GNNs. Conversely, LLaGA exhibits significantly longer training time due to the requirement of batch processing.

\section{Limitation}\label{sec:app-limitation}
This paper proposes an adaptive recurrent graph model, AdaR, enabling flexible test-time computation on graphs. In the current implementation, AdaR still operates under a pre-defined maximum iteration budget $S$. In practice, we set $S=300$ for zero-shot evaluation, while allowing early termination based on the average inner-product similarity between node representations and targets. Despite this practical flexibility, determining an input-adaptive criterion that automatically infers the optimal iteration budget remains an open problem, which we leave for future work.

%%%%%%%%%%%%%%%%%%%%%%%%%%%%%%%%%%%%%%%%%%%%%%%%%%%%%%%%%%%%%%%%%%%%%%%%%%%%%%%
%%%%%%%%%%%%%%%%%%%%%%%%%%%%%%%%%%%%%%%%%%%%%%%%%%%%%%%%%%%%%%%%%%%%%%%%%%%%%%%

\end{document}